\documentclass[english,11pt]{article}
 
\makeatletter

\usepackage{amsfonts,amssymb}
\usepackage{titlesec}
\usepackage{titletoc}
\usepackage{amsmath}
\usepackage{listings}
\usepackage{verbatim}
\usepackage{graphicx}
\usepackage{geometry}
\usepackage[colorlinks=true,linkcolor=black,citecolor=red,urlcolor=green]{hyperref}
\usepackage{mathrsfs}
\usepackage{setspace}
\usepackage{cite}
\usepackage{bbm}
\usepackage{mathtools}
\usepackage{geometry} 
\usepackage{listings}

\makeatletter
\newcommand{\manuallabel}[2]{\def\@currentlabel{#2}\label{#1}}
\makeatother

\numberwithin{equation}{section}
\usepackage{amsthm}
\usepackage{xcolor, tikz}
\newtheorem{theorem}{Theorem}

\newtheorem{lem}{Lemma}
\newtheorem{definition}{Definition}
\newtheorem{rem}{Remark}
\usepackage{bm}

\DeclareMathOperator{\sign}{sign}

\DeclareMathOperator{\rank}{rank}
\DeclareMathOperator{\rad}{rad}
\DeclareMathOperator{\tb}{2b}
\DeclareMathOperator{\mb}{mb}
\DeclareMathOperator{\diag}{diag}

\DeclareMathOperator{\tr}{\mathrm{Tr}}

\title{A Parameter-Free Two-Bit   Covariance   Estimator\\ with Improved Operator Norm Error Rate}
\author{Junren Chen\thanks{J. Chen is with Department of Mathematics, The University of Hong Kong. (email: \texttt{chenjr58@connect.hku.hk})}\and Michael K. Ng\thanks{M. K. Ng is with Department of Mathematics, Hong Kong Baptist University. (email: \texttt{michael-ng@hkbu.edu.hk})}}
\date{(\today)}
\begin{document}
    \maketitle
\begin{abstract}
     A covariance matrix estimator using {two bits per entry} was recently developed by Dirksen, Maly and Rauhut  
     [Annals of Statistics, 50(6), pp. 3538-3562]. The  estimator achieves near minimax operator norm rate for general sub-Gaussian distributions,  
     but also suffers from two downsides: theoretically, there is an essential gap on operator norm error between their estimator and sample covariance when the diagonal of the covariance matrix is dominated by only a few entries; practically, its performance heavily relies on  the dithering scale, which needs to be tuned according to some unknown parameters. In this work, we propose a new 2-bit covariance matrix estimator that simultaneously addresses both issues. Unlike the sign quantizer associated with uniform dither   in Dirksen {\em et al.}, we adopt a triangular dither prior to a 2-bit quantizer inspired by the multi-bit uniform quantizer. By employing dithering scales varying across entries, our estimator  enjoys an improved operator norm error rate that depends on the effective rank of the underlying covariance matrix rather than the ambient dimension {\color{black}(referred to as dimension-free bound)}, which matches that of sample covariance up to logarithmic factors. Moreover, our proposed method eliminates the need of {\it any} tuning parameter, as the dithering scales are entirely determined by the data. {\color{black}While our estimator requires a pass of all unquantized samples to determine the dithering scales, it can be adapted to the online setting where the samples arise sequentially.}
     Experimental results are provided to demonstrate the advantages of our estimators over the existing ones.  
\end{abstract}
\section{Introduction}
Given i.i.d. samples $\bm{X}_1,...,\bm{X}_n$     from a zero-mean random vector $\bm{X}$, a fundamental problem in multivariate analysis is to accurately estimate the covariance matrix $\bm{\Sigma}:=\mathbbm{E}(\bm{X}\bm{X}^\top)$, which frequently arises in  principle component analysis \cite{jolliffe2002principal}, regression analysis \cite{freedman2009statistical}, finance \cite{ledoit2003improved}, massive MIMO system \cite{maly2022new}, and so on. Despite the vast literature on covariance estimation, we study in this paper a less well-understood setting   where we can only access samples coarsely quantized to a small number of bits. This setting is of particular interest in signal processing  or distributed learning where it   is expensive or even impossible to acquire or transmit high-precision data (e.g., \cite{yang2023plug,chen2022high,danaee2022distributed,hanna2021quantization,zhang2017zipml,maly2022new}). Defining the $\psi_2$ norm and $L^2$ norm of a random variable $X$ respectively as $\|X\|_{\psi_2}=\inf\{t>0:\mathbbm{E}(\exp(\frac{X^2}{t^2}))\leq 2\}$ and $\|X\|_{L^2}=(\mathbbm{E}[X^2])^{1/2}$, we follow prior works (e.g., \cite{12-BEJHDcovariance,koltchinskii2017concentration,vershynin2018high,vershynin_2012}) and formulate sub-Gaussian random vector $\bm{X}$ in Definition \ref{def1}.   
\begin{definition}\label{def1}
    A  $d$-dimensional random vector $\bm{X}$ is $K$-sub-Gaussian if 
    \begin{equation}
        \label{KSG}
        \|\langle\bm{X},\bm{v}\rangle\|_{\psi_2}\leq K\|\langle\bm{X},\bm{v}\rangle\|_{L^2},~\forall~\bm{v}\in \mathbb{R}^{d}. 
    \end{equation}
    For given $K$, we use the shorthand $\hat{K}:=K\sqrt{\log K}$ throughout this work.
\end{definition}
      In this paper, we propose a 2-bit covariance estimator\footnote{Throughout this work, 2-bit estimator refers to the one that only relies on two bits from each entry of $\{\bm{X}_i\}_{i=1}^n$.} for general sub-Gaussian $\bm{X}$ and establish  {\it non-asymptotic guarantee} which indicates  the  required sample size   for achieving operator norm error lower than some desired accuracy.     
 
\subsection{Related Works}
Under a direct 1-bit sign quantizer, many works have investigated covariance estimation from $\{\sign(\bm{X}_i)\}_{i=1}^n$\footnote{Here, $\sign(\cdot)$ is the 1-bit (sign) quantizer defined as $\sign(a)=1$ (if $a\geq 0$) and $\sign(a)=-1$ (if $a<0$).} under the name of  ``arcsin-law" (e.g., \cite{jacovitti1994estimation,choi2016near,van1966spectrum,vershynin2018high}), but it is impossible to estimate the full covariance in this case as the magnitudes of the samples are completely lost.\footnote{Indeed, we can only hope to estimate the correlation matrix with all-ones diagonal (e.g., \cite{dirksen2022covariance}), {\color{black}and in this case the {\it non-asymptotic} error rate has  also been established under {\it Gaussian} samples \cite[Thm. 1]{dirksen2022covariance}}.} For estimating the full  covariance $\bm{\Sigma}$, as we pursue in this paper, one must introduce non-zero thresholds $\bm{\tau}_i \in \mathbb{R}^n$ called ``dither" before the 1-bit quantization and collect $\sign(\bm{X}_i+\bm{\tau}_i)$ instead. Such a technique   known as ``dithering", with its aid on signal reconstruction observed in engineering works \cite{roberts1962picture,limb1969design,jayant1972application} and theory established in \cite{gray1993dithered,schuchman1964dither},   recently again received much attention from researchers working on various estimation problems like compressed sensing, covariance estimation, matrix completion and reduced-rank regression (e.g., \cite{chen2022high,chen2023quantized,dirksen2019quantized,dirksen2021non,dirksen2022covariance,xu2020quantized,thrampoulidis2020generalized,jung2021quantized,chen2022quantizing}). Specifically, for covariance estimation, some works proposed to use random Gaussian dithering \cite{eamaz2021modified,eamaz2022covariance,eamaz2023covariance} or deterministic fixed dithering (e.g., \cite{chapeau2008fisher,fang2010adaptive}), but to our best knowledge, these two lines of works are restricted to {\it Gaussian samples and  do not establish non-asymptotic result}, thus provide little guarantee for the practical finite-sample setting. We will not further explain these works since our interest is on estimators for {\it sub-Gaussian} distributions with {\it non-asymptotic guarantee}.

The work most relevant to ours is Dirksen {\em et al.} \cite{dirksen2022covariance} who developed the first 2-bit covariance estimator that applies to sub-Gaussian samples and enjoys non-asymptotic near minimax error rate.
 Let $\{\bm{\tau}_{i1},\bm{\tau}_{i2}\}_{i=1}^n\sim\mathscr{U}[-\lambda,\lambda]^d$ be i.i.d. uniform dithers independent of the $K$-sub-Gaussian    $\{\bm{X}_i\}_{i=1}^n$, they proposed to quantize the $i$-th sample $\bm{X}_i$ to $\bm{\dot{X}}_{i1}=\lambda\cdot\sign(\bm{X}_i+\bm{\tau}_{i1})$ and $\bm{\dot{X}}_{i2}=\lambda\cdot\sign(\bm{X}_i+\bm{\tau}_{i2})$, thus only retaining two bits per entry. 
Based on the observation  $\mathbbm{E}(\bm{\dot{X}}_{i1}\bm{\dot{X}}_{i2}^\top)=\bm{\Sigma}$ that holds if  $\lambda\geq \|\bm{X}_i\|_\infty$ \cite[Lem. 15]{dirksen2022covariance}, the 2-bit estimator in \cite{dirksen2022covariance} is defined as 
\begin{equation}\label{aosesti}
\begin{aligned}
    \bm{\widehat{\Sigma}} _{na}= \frac{1}{2n}\sum_{i=1}^n (\bm{\dot{X}}_{i1}\bm{\dot{X}}_{i2}^\top+\bm{\dot{X}}_{i2}\bm{\dot{X}}_{i1}^\top),\quad\text{with }\begin{cases}
        \bm{\dot{X}}_{i1}=\lambda\cdot \sign(\bm{X}_i+\bm{\tau}_{i1})\\
        \bm{\dot{X}}_{i2}=\lambda\cdot \sign(\bm{X}_i+\bm{\tau}_{i2})
    \end{cases},
    \end{aligned}
\end{equation}
where $\{\bm{\tau}_{i1},\bm{\tau}_{i2}\}_{i=1}^n\sim \mathscr{U}[-\lambda,\lambda]^d$.
The non-masked case of \cite[Thm. 4]{dirksen2022covariance} implies the following operator norm error rate that is in general near minimax optimal up to logarithmic factors (e.g., see their supplementary material \cite{dirksen2022supp}).  

\begin{theorem}\label{thm1}
    Suppose that $\bm{X}_1,...,\bm{X}_n$ are i.i.d. copies of the zero-mean $K$-sub-Gaussian random vector $\bm{X}\in \mathbb{R}^d$, we consider the 2-bit estimator $\bm{\widehat{\Sigma}}_{na}$ in (\ref{aosesti}) with uniform dithers $\bm{\tau}_{i1},\bm{\tau}_{i2}\sim \mathscr{U}[-1,1]^d$. If $\lambda^2= C(K)\log n\|\bm{\Sigma}\|_\infty$ for some large enough constant $C(K)$ depending on $K$, 
   then for any $t>0$, it holds with probability at least $1-e^{-t}$ that 
   $$\|\bm{\widehat{\Sigma}}_{na}- \bm{\Sigma}\| _{op}\lesssim _K \sqrt{\frac{d\|\bm{\Sigma}\|_{op}\|\bm{\Sigma}\|_{\infty}[\log (d)+t](\log n)^2}{n}}+\frac{d\|\bm{\Sigma}\|_{\infty}[\log(d)+t]\log(n)}{n}.$$
  Setting $t = 10\log(nd)$ yields that, with probability at least $1-2(nd)^{-10}$ that 
    \begin{equation}\label{1.71}
        \|\bm{\widehat{\Sigma}}_{na}- \bm{\Sigma}\| _{op}\lesssim _K  \sqrt{\frac{d\|\bm{\Sigma}\|_{op}\|\bm{\Sigma}\|_\infty\log (nd) (\log n)^2}{n}}+ \frac{d\|\bm{\Sigma}\|_\infty \log n\log (nd)}{n}.
    \end{equation}
    Moreover, under the scaling of $n\gtrsim _K  \frac{d\|\bm{\Sigma}\|_\infty\log (nd)}{\|\bm{\Sigma}\|_{op}}$, $\|\bm{\widehat{\Sigma}}_{na}- \bm{\Sigma}\| _{op}\lesssim _K  \sqrt{\frac{d\|\bm{\Sigma}\|_{op}\|\bm{\Sigma}\|_\infty\log(nd)(\log n)^2}{n}}$ holds with high probability. 
\end{theorem}
The novel estimator in (\ref{aosesti}) represents a significant progress in covariance estimation, which has    already sparked interest and led to subsequent research endeavors: Chen {\it et al.} extended it to high-dimensional sparse case {\color{black}(with unknown sparsity pattern)} and heavy-tailed samples   in \cite[Sec. II]{chen2022high}, and also developed a multi-bit covariance estimator for heavy-tailed distribution with near minimax non-asymptotic guarantees in \cite[Sec. 3.1]{chen2022quantizing}; Yang {\it et al.} \cite{yang2023plug} extended it to the complex domain and applied to massive MIMO system; more recently, Dirksen and Maly proposed its tuning-free version by using data-driven dithering \cite{dirksen2023tuning}.

\subsection{Two Downsides of the Estimator  of \cite{dirksen2022covariance}}
Nonetheless, the estimator $\bm{\widehat{\Sigma}}_{na}$ also suffers from two frustrating downsides.

\paragraph{Theoretical Gap.} Although 
the rate in (\ref{1.71}) is near minimax over certain set of covariance matrices, it is essentially sub-optimal for covariance matrices whose diagonals are dominated by only a few entries, which can be mathematically captured  by $\tr(\bm{\Sigma})\ll d\|\bm{\Sigma}\|_\infty$. Indeed,
for  the behaviour of the sample covariance $\bm{\widehat{\Sigma}}=\frac{1}{n}\sum_{i=1}^n\bm{X}_i\bm{X}_i^\top$ on sub-Gaussian $\bm{X}_i$, the {\it dimension-free} operator norm error rate $$ O\left(\|\bm{\Sigma}\|_{op}\Big[\sqrt{\frac{r(\bm{\Sigma})}{n}}+\frac{r(\bm{\Sigma})}{n}\Big]\right)$$ depending on the {\it effective rank}  $r(\bm{\Sigma}):=\frac{\tr(\bm{\Sigma})}{\|\bm{\Sigma}\|_{op}}$, which could be much smaller than the ambient dimension $d$, has been established in the literature, see for instance  \cite[Prop. 3]{12-BEJHDcovariance}, \cite[Coro. 2]{koltchinskii2017concentration}, \cite[Thm. 9.2.4]{vershynin2018high}. This rate is   tight up to multiplicative constant for Gaussian samples \cite[Thm. 4]{koltchinskii2017concentration} and also optimal in a  minimax sense \cite[Thm. 2]{12-BEJHDcovariance}. To be more precise, we present Theorem \ref{lem3} which provides an incremental extension of \cite[Thm. 9.2.4]{vershynin2018high} towards independent samples that  may not be identically distributed (such generality will be useful in the proof of our main theorem). Note that it also slightly refines the dependence on $K$ by the recent results from  \cite{jeong2022sub}.

\begin{theorem}
    \label{lem3}
     Suppose that   $\bm{X}_1,...,\bm{X}_n$ are independent, zero-mean, $K$-sub-Gaussian $d$-dimensional random vectors sharing the same covariance matrix $\bm{\Sigma}\in \mathbb{R}$. Let $r(\bm{\Sigma}):=\frac{\tr(\bm{\Sigma})}{\|\bm{\Sigma}\|_{op}}$, then for any $u\geq 0$, with probability at least $1-3e^{-u}$ we have \begin{equation}\label{1.444}
        \left\|\frac{1}{n}\sum_{i=1}^n\bm{X}_i\bm{X}_i^\top-\bm{\Sigma}\right\|_{op}\leq C  \|\bm{\Sigma}\|_{op}\left(\sqrt{\frac{\hat{K}^2(r(\bm{\Sigma})+u)}{n}}+\frac{\hat{K}^2(r(\bm{\Sigma})+u)}{n}\right).
    \end{equation}
\end{theorem}
\begin{proof}
    The proof can be found in Appendix A.
    \end{proof}
Therefore, setting $u=10r(\bm{\Sigma})$ in (\ref{1.444}), 
the sample covariance satisfies a high-probability operator norm error bound $O\big(\sqrt{\frac{\tr(\bm{\Sigma})\|\bm{\Sigma}\|_{op}}{n}}+\frac{\tr(\bm{\Sigma})}{n}\big)$ that is essentially tighter than (\ref{1.71}) when $\tr(\bm{\Sigma})\ll d\|\bm{\Sigma}\|_{\infty}$. To the best of our knowledge, there has not been any attempt to close this gap between 2-bit covariance and the sample covariance based on full data.

\paragraph{Tuning Parameter.} 
More practically, while the behaviour of $\bm{\widehat{\Sigma}}_{na}$ heavily relies on a suitable dithering scale $\lambda$, this parameter is in general hard to tune, e.g.,  its theoretical choice depends on $K$ and $\|\bm{\Sigma}\|_\infty$. While $K$ is related to the sub-Gaussianity of the underlying distribution, the dependence on $\|\bm{\Sigma}\|_\infty$ appears more undesired since $\|\bm{\Sigma}\|_\infty$ could significantly vary in different instances. To alleviate this tuning issue, the recent work \cite{dirksen2023tuning} proposed a tuning-free variant of $\bm{\widehat{\Sigma}}_{na}$ that uses data-driven dithering to get rid of the dependence of dithering scales on $\|\bm{\Sigma}\|_\infty$. Particularly,  they added an additional sample $\bm{X}_0$ and   then used a uniform dither with scale $$\lambda_i=C(K)\left(\frac{1}{i}\sum_{k=0}^{i-1}\|\bm{X}_k\|_\infty\right)\sqrt{\log i}$$ for the $i$-th sample $\bm{X}_i$ before the 1-bit quantization, where $C(K)$ is a fixed constant depending only on  $K$ (but not on $\|\bm{\Sigma}\|_\infty$). More precisely, let $\{\bm{\tau}_{i1},\bm{\tau}_{i2}\}_{i=1}^n\stackrel{iid}{\sim}\mathscr{U}[-1,1]^d$ be independent of $\{\bm{X}_i\}_{i=0}^n$,  analogously to the construction in (\ref{aosesti}), their estimator is given by \footnote{Compared to $\bm{\widehat{\Sigma}}_{na}$, this estimator requires storing the additional $\{\lambda_i\}_{i=1}^n$ as 32-bit   data, while such memory is relatively minor compared to the savings of quantization \cite[Rem. 2]{dirksen2023tuning}; this remark also applies to the our main estimator $\bm{\widetilde{\Sigma}}$ in Section \ref{sec3} that additionally requires    $d$   full-precision scalars.}  $$\bm{\widehat{\Sigma}}_a=\frac{1}{2n}\sum_{i=1}^n(\bm{\dot{X}}_{i1}\bm{\dot{X}}_{i2}^\top+\bm{\dot{X}}_{i2}\bm{\dot{X}}_{i1}^\top),\quad\text{with }\begin{cases}
    \bm{\dot{X}}_{i1}=\lambda_i\cdot \sign(\bm{X}_i+\lambda_i\bm{\tau}_{i1})\\
    \bm{\dot{X}}_{i2}=\lambda_i\cdot\sign(\bm{X}_i+ \lambda_i\bm{\tau}_{i2})
\end{cases},$$
where $\{\bm{\tau}_{i1},\bm{\tau}_{i2}\}_{i=1}^n\sim \mathscr{U}[-1,1]^d$.
The main theorem of \cite{dirksen2023tuning} implies that \begin{equation}\label{guaranteeSigmaa}
    \|\bm{\widehat{\Sigma}}_{a}-\bm{\Sigma}\|_{op}\lesssim _K\sqrt{\frac{d\|\bm{\Sigma}\|_\infty\|\bm{\Sigma}\|_{op}\log^7(nd)\log^2n}{n}}+\frac{d\|\bm{\Sigma}\|_\infty \log^4 (nd)\log n}{n}
\end{equation}
holds with probability at least $1-2(nd)^{-10}$, which coincides with (\ref{1.71}) up to logarithmic factors. Nonetheless, $\bm{\widehat{\Sigma}}_{a}$ does not fully address the tuning issue since the  unknown multiplicative constant still depends on $K$ that could vary among different distributions. Besides, as will be shown by the numerical examples below,   $\bm{\widehat{\Sigma}}_a$ does not perform well empirically.

\subsubsection{Numerical Examples}
We pause to provide numerical results to demonstrate the aforementioned   weaknesses of existing 2-bit estimators. We fix $(n,d)=(500,10)$ and test Gaussian samples $\bm{X}_i\sim \mathcal{N}(0,\bm{\Sigma})$, for which the dithering scale is $\lambda=C\sqrt{\log n}$ in $\bm{\widehat{\Sigma}}_{na}$ and $\lambda_i=C \big( \frac{1}{i}\sum_{k=1}^{i-1}\|\bm{X}_k\|_\infty\big) \sqrt{\log i}$ in $\bm{\widehat{\Sigma}}_{a}$. Based on $50$ repetitions, we plot the curves of ``operator norm error v.s. $C$" in Figure \ref{fig1}, where we also include the results of sample covariance $\bm{\widehat{\Sigma}}$ and a new (non-adaptive) 2-bit estimator $\widetilde{\bm{\Sigma}}_{\tb}$ that we propose in Section \ref{sec2}. Let $$\bm{\Sigma}(a,b,c)=(a-b)\bm{I}_d+b\mathbf{1}\mathbf{1}^\top+(c-a)\bm{e}_1\bm{e}_1^\top  = \begin{bmatrix}
    c & b & b & \cdots & b \\
    b & a & b &\cdots & b \\
    b& b & a &\cdots & b \\
    \vdots &\vdots & \vdots & \ddots & \vdots \\
    b & b & b &\cdots & a 
\end{bmatrix}$$ be the covariance matrix with diagonal   being $[c,a,...,a]^\top$ and non-diagonal entries being $b$, we test the low-correlation case $\bm{\Sigma}=\bm{\Sigma}(1,0.2,1)$ in Figure \ref{fig1}(a) and the high-correlation case $\bm{\Sigma}=\bm{\Sigma}(1,0.9,1)$ in Figure \ref{fig1}(b). Further, we test $\bm{\Sigma}(1,0.2,10)$ (that changes the $(1,1)$-th entry of $\bm{\Sigma}(1,0.2,1)$ to 10) in Figure \ref{fig1}(c) to simulate the setting of $\tr(\bm{\Sigma})\ll d\|\bm{\Sigma}\|_\infty$. Comparing (a) and (c) in Figure \ref{fig1} clearly corroborates the performance gap between $\bm{\widehat{\Sigma}}_{na}$ and $\bm{\widehat{\Sigma}}$ under $\tr(\bm{\Sigma})\ll d\|\bm{\Sigma}\|_\infty$, as well as the dependence of ``optimal $C$" (for $\bm{\widehat{\Sigma}}_{na}$) on $\|\bm{\Sigma}\|_\infty$. Consistent with the theoretical progress made by \cite{dirksen2023tuning}, the optimal $C$ for $\bm{\widehat{\Sigma}}_a$ roughly remain in $[0.3,0.4]$, but its numerical performance is in general worse than $\bm{\widehat{\Sigma}}_{na}$, especially when the correlations are high  or $\tr(\bm{\Sigma})\ll d\|\bm{\Sigma}\|_\infty$ (Figure \ref{fig1}(b)-(c)).\footnote{\color{black}But we also note that this observation is reached under the circumstance where the non-adaptive estimators $\widehat{\bm{\Sigma}}_{\rm na}$ and $\widetilde{\bm{\Sigma}}_{\rm 2b}$ are well tuned; the tuning advantage of $\widehat{\bm{\Sigma}}_a$ should be gladly admitted. } We conjecture that this numerical degradation stems from the scaling $\sqrt{\log i}$   in the $\lambda_i$, which may facilitate theoretical analysis but could make the dithering scale of $\bm{X}_{i_1}$ considerably larger than $\bm{X}_{i_2}$ if $i_1\gg i_2$; this may not be sensible since there is   no reason to believe that the entries of $\bm{X}_{i_1}$ are of magnitudes greater than those of $\bm{X}_{i_2}$.

\begin{figure}[ht!]
    \centering
    \includegraphics[scale = 0.50]{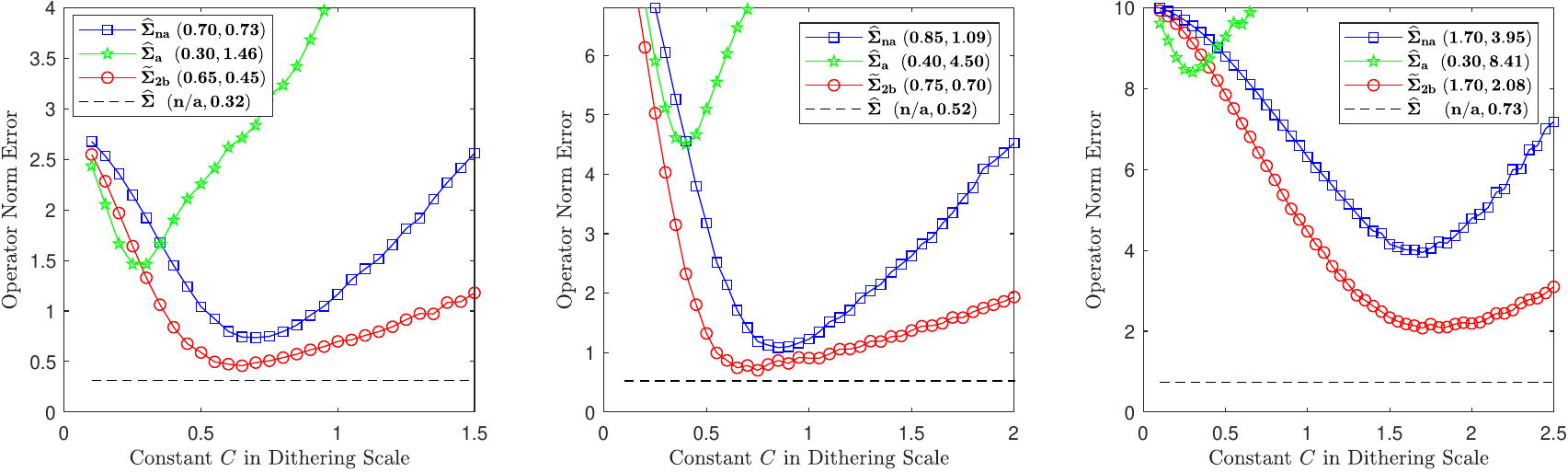}

     ~~~   (a) \hspace{4.3cm} (b)  \hspace{4.3cm} (c) 
    \caption{\footnotesize The curves of ``operator norm error v.s. $C$", with  the optimal $C$  and the corresponding minimum error  reported in the labels. We simulate $\bm{\Sigma} = \bm{\Sigma}(1,0.2,1)$ in (a), $\bm{\Sigma}=\bm{\Sigma}(1,0.9,1)$ in (b), and $\bm{\Sigma}=\bm{\Sigma}(1,0.2,10)$ in (c).} 
    \label{fig1}
\end{figure}
\subsection{Our Contributions}

In this work, we develop a new 2-bit covariance estimator to address the above two issues. Our quantization procedure is an essential departure from \cite{dirksen2023tuning,dirksen2022covariance} in the sense that we use triangular dither and a 2-bit quantizer (see (\ref{2.15}) below) reduced from the multi-bit uniform quantizer (see (\ref{uniformquan}) below). Indeed, our main estimator is built upon a new non-adaptive 2-bit estimator $\bm{\widetilde{\Sigma}}_{\tb}$ inspired by the multi-bit estimator developed in our prior work \cite{chen2022quantizing}; compared to $\bm{\widehat{\Sigma}}_{na}$, our new 2-bit estimator $\bm{\widetilde{\Sigma}}_{\tb}$ has comparable theoretical guarantee  and better numerical performance over Gaussian samples (see Figure \ref{fig1}). Then, we modify the proposed $\bm{\widetilde{\Sigma}}_{\tb}$ by using dithering scales that vary across different coordinates (i.e., entries) of $\bm{X}_i$ and are {\it entirely specified} by the given data. In particular, we use the largest magnitude of the $j$-th entries of all samples as the dithering scale for this coordinate; see $\lambda_{j,\min}$ in (\ref{3.1}) below.
This not only removes the tuning parameter but also allows for a dimension-free near-optimal operator norm error bound---our main theorem   presents a high-probability operator norm error rate $$O_K\left(\|\bm{\Sigma}\|_{op}\Big(\sqrt{\frac{r(\bm{\Sigma})\log ^2(nd)}{n}}+\frac{r(\bm{\Sigma})\log^2(nd)}{n}\Big)\right)$$ that matches   Theorem \ref{lem3} for sample covariance up to (a small number of) logarithmic factors. {\color{black}While determining the dithering scales $(\lambda_{j,\min})_{j=1}^d$ needs one pass of all unquantized samples, we also discuss  how to adjust our method to the online setting where the samples arise sequentially, see  $\widetilde{\bm{\Sigma}}_{on}$ (which achieves the same dimension-free error rate 
and requires tuning a constant depending only on $K$) in Section \ref{sec6} for instance.}

We conclude this section by fixing notations and providing an overview.

\paragraph{Notations.} We use regular letters to denote scalars, and boldface symbols are for vectors and matrices. For vector $\bm{a}=(a_i)\in\mathbb{R}^d$ we let $\|\bm{a}\|_\infty=\max_{1\leq i\leq d}|a_i|$ and $\|\bm{a}\|_2=(\sum_{i=1}^da_i^2)^{1/2}$; for matrix $\bm{A}=(a_{ij})\in \mathbb{R}^{d\times d}$ we let $\|\bm{A}\|_{op}$ be the operator norm and $\|\bm{A}\|_\infty=\max_{i,j\in[d]}|a_{ij}|$. Given symmetric $\bm{A},\bm{B}\in \mathbb{R}^{d\times d}$, we write $\bm{A}\preceq \bm{B}$ if $\bm{B}-\bm{A}$ is positive semi-definite.  Recall that $\|X\|_{\psi_2}$ has been defined for a random variable $X$ as $\|X\|_{\psi_2}=\inf\{t>0:\mathbbm{E}(\exp(\frac{X^2}{t^2}))\leq 2\}$ and $\|X\|_{L^2}=(\mathbbm{E}[X^2])^{1/2}$, now we further define $\|\bm{X}\|_{\psi_2}:=\sup_{\bm{v}\in \mathbb{S}^{d-1}}\|\langle \bm{v},\bm{X} \rangle\|_{\psi_2}$ for a random vector $\bm{X}\in \mathbb{R}^d
$, where $\mathbb{S}^{d-1}=\{\bm{v}\in \mathbb{R}^d:\|\bm{v}\|_2=1\}$ is the standard Euclidean sphere, $\langle\bm{a},\bm{b}\rangle=\bm{a}^\top\bm{b}$ is the inner product. 
We  use $C,c,C_1,C_2,...$ to denote absolute constants whose values may vary from line to line. We   denote the zero-mean multivariate Gaussian distribution with covariance matrix $\bm{\Sigma}$ by $\mathcal{N}(0,\bm{\Sigma})$, and   the uniform distribution over   $W\subset \mathbb{R}^d$ by  $\mathscr{U}(W)$.
Given two terms $T_1$ and $T_2$, we write $T_1\lesssim T_2$ or $T_1=O(T_2)$ if $T_1\leq CT_2$ for some $C$, and conversely write $T_1\gtrsim T_2$ or $T_1=\Omega(T_2)$ if $T_1\geq cT_2$ for some $c$. We use $T_1\asymp T_2$ to state that $T_1=O(T_2)$ and $T_1=\Omega(T_2)$ simultanesouly hold. We may also use notations such as $\lesssim_K,~\gtrsim _K,~O_K(\cdot)$ to indicate that the implied constants can depend on $K$.  In this work, the quantizers operate
on vectors in an element-wise manner. In our problem setting, we let $X_{ij}$ be the $j$-th entry of the sample $\bm{X}_i$, $\Sigma_{ij}$ be the $(i,j)$-th entry of the underlying covariance $\bm{\Sigma}$.

\paragraph{Overview.} We propose a new 2-bit estimator $\bm{\widetilde{\Sigma}}_{\tb}$ in Section \ref{sec2} that is comparable to $\bm{\widehat{\Sigma}}_{na}$. Our primary contribution, a parameter-free estimator with improved error rate, is presented in Section \ref{sec3}, where we also provide numerical examples to validate our theory. The proof of our main theorem (Theorem \ref{maintheorem}) is given in Section \ref{sec4}.  The paper is concluded with some discussions in Section \ref{sec6}. We  collect the deferred proofs in Appendices A-C that follow the concluding section.

\section{A New 2-bit Covariance Matrix Estimator}\label{sec2}
\subsection{A Multi-Bit Covariance Estimator} 
Using the uniform quantizer associated with triangular dithers, a multi-bit covariance estimator was developed in \cite[Sec. 3.1]{chen2022quantizing}. We shall introduce this estimator (denoted by $\bm{\widetilde{\Sigma}}_{\mb}$) first in order to inspire the new 2-bit estimator $\bm{\widetilde{\Sigma}}_{\tb}$.

 We need some necessary preliminaries of the dithered uniform quantizer to get started. The (multi-bit) uniform quantizer with resolution $\lambda$ is defined for a real scalar as \begin{equation}\label{uniformquan}
     \mathcal{Q}_\lambda(a)=\lambda\Big(\Big\lfloor \frac{a}{\lambda}\Big\rfloor+\frac{1}{2}\Big),~~a\in \mathbb{R}.
 \end{equation} It is now well-understood that dithering prior to uniform quantization benefits signal recovery (or parameter estimation) from different aspects (e.g., \cite{chen2022quantizing,sun2022quantized,xu2020quantized,chen2023quantized,thrampoulidis2020generalized,jung2021quantized}); throughout this work, the dithering vectors (i.e., dithers) $\bm{\tau}_i$'s are independent of everything else and have independent entries.\footnote{For data-driven dithers with dithering scale depending on $\bm{X}_i$'s, we  will describe the procedure as drawing $\bm{\tau}_i$'s with constant scale (i.e., 1) and then rescaling the dithers (e.g., see our description for $\bm{\widehat{\Sigma}}_{a}$ above).}  Under uniform quantizer $\mathcal{Q}_\lambda(\cdot)$, while a uniform dither $\bm{\tau}_i\sim \mathscr{U}[-\frac{\lambda}{2},\frac{\lambda}{2}]^d$ is adopted in most prior works,
 the triangular dither defined as the sum of two {\it independent} uniform dithers \begin{equation}
     \label{triangulard}
     \bm{\tau}_i\sim \mathscr{U}\Big[-\frac{\lambda}{2},\frac{\lambda}{2}\Big]^d+\mathscr{U}\Big[-\frac{\lambda}{2},\frac{\lambda}{2}\Big]^d
 \end{equation} was adopted in \cite{chen2022quantizing,chen2023quantized} to allow for covariance estimation. Our subsequent developments hinge on the following properties of the dithered quantizer $\mathcal{Q}_\lambda(\cdot+\bm{\tau})$ with triangular dither. 

 \begin{lem}\label{lem1}
     For deterministic $\bm{X}\in \mathbb{R}^d$ and some $\lambda>0$, we let $\bm{\tau}\sim \mathscr{U}[-\frac{\lambda}{2},\frac{\lambda}{2}]^d+\mathscr{U}[-\frac{\lambda}{2},\frac{\lambda}{2}]^d$ be independent of $\bm{X}$ and quantize $\bm{X}$ to $\bm{\dot{X}}=\mathcal{Q}_\lambda(\bm{X}+\bm{\tau})$. Then the quantization noise $\bm{\xi}=\bm{\dot{X}}-\bm{X}$ satisfies the following: (i) $\|\bm{\xi}\|_\infty \leq \frac{3\lambda}{2}$, (ii) $\mathbbm{E}[\bm{\xi}]=0$, (iii) $\|\bm{\xi}\|_{\psi_2}=O(\lambda)$, and (iv) $\mathbbm{E}(\bm{\xi}\bm{\xi}^\top)=\frac{\lambda^2}{4}\bm{I}_d$. 
 \end{lem}
\begin{proof}
    The proof can be found in Appendix B. 
\end{proof}

Lemma \ref{lem1} allows for a covariance estimator based on $\mathcal{Q}_\lambda(\cdot+\bm{\tau})$ with triangular dither $\bm{\tau}$. Specifically, we let $\{\bm{\tau}_i\}_{i=1}^n\stackrel{iid}{\sim}\mathscr{U}[-\frac{\lambda}{2},\frac{\lambda}{2}]^d+\mathscr{U}[-\frac{\lambda}{2},\frac{\lambda}{2}]^d$ be independent of $\{\bm{X}_i\}_{i=1}^n$, and then quantize $\bm{X}_i$ to $\bm{\dot{X}}_i=\mathcal{Q}_\lambda(\bm{X}_i+\bm{\tau}_i)$. Define $\bm{\xi}_i:=\bm{\dot{X}}_i-\bm{X}_i$, then we have 
\begin{equation}\begin{aligned}\label{2.1}
    \mathbbm{E}(\bm{\dot{X}}_i\bm{\dot{X}}_i^\top) &= \mathbbm{E}(\bm{X}_i\bm{X}_i^\top)+\mathbbm{E}(\bm{\xi}_i\bm{\xi}_i^\top)+\mathbbm{E}(\bm{X}_i\bm{\xi}_i^\top)+\mathbbm{E}(\bm{\xi}_i\bm{X}_i^\top)=\bm{\Sigma}+\frac{\lambda^2}{4}\bm{I}_d,
    \end{aligned}
\end{equation}
where we use $\mathbbm{E}(\bm{\xi}_i\bm{\xi}_i^\top)=\frac{\lambda^2}{4}\bm{I}_d$ and $\mathbbm{E} (\bm{\xi}_i)=0$ (conditioning on $\bm{X}_i$) from Lemma \ref{lem1} in the second equality. Equation (\ref{2.1}) motivates the estimator \cite{chen2022quantizing}  
\begin{equation}\label{2.4}
    \bm{\widetilde{\Sigma}}_{\mb}= \frac{1}{n}\sum_{i=1}^n \bm{\dot{X}}_i\bm{\dot{X}}_i^\top- \frac{\lambda^2}{4}\bm{I}_d,\quad\text{with }\bm{\dot{X}}_i=\mathcal{Q}_\lambda(\bm{X}_i+\bm{\tau}_i),
\end{equation}
where $\{\bm{\tau}_i\}_{i=1}^n\sim \mathscr{U}[-\frac{\lambda}{2},\frac{\lambda}{2}]^d+\mathscr{U}[-\frac{\lambda}{2},\frac{\lambda}{2}]^d$.
Note that \cite[Sec. 3.1]{chen2022quantizing} focused on heavy-tailed $\bm{X}_i$ (assumed to have bounded fourth moments) and overcame the heavy-tailedness by incorporating an additional truncation step before the dithered quantization. As a result, it was  mentioned in Remark 1 therein that ``the   results for sub-Gaussian distributions can be analogously established and are also new to the literature." 
To facilitate further developments, we present here the sub-Gaussian counterpart of \cite[Thm. 3]{chen2022quantizing}. 

\begin{theorem}\label{thmmb}
    Suppose that   $\bm{X}_1,...,\bm{X}_n$ are i.i.d. copies of the zero-mean $K$-sub-Gaussian random vector $\bm{X}\in \mathbb{R}^d$. Given $\lambda>0$, using the triangular dithers $\bm{\tau}_i\sim \mathscr{U}[-\frac{\lambda}{2},\frac{\lambda}{2}]^d+\mathscr{U}[-\frac{\lambda}{2},\frac{\lambda}{2}]^d$, the estimator $\bm{\widetilde{\Sigma}}_{\mb} $ in (\ref{2.4})  satisfies \begin{equation}\nonumber
        \big\|\bm{\widetilde{\Sigma}}_{\mb}-\bm{\Sigma}\big\|_{op}\leq C\Big(\|\bm{\Sigma}\|_{op}+\frac{\lambda^2}{4}\Big)
\left(\sqrt{\frac{\hat{K}^2\big[r(\bm{\Sigma}+\frac{\lambda^2}{4}\bm{I}_d)+u\big]}{n}}+ \frac{\hat{K}^2\big[r(\bm{\Sigma}+\frac{\lambda^2}{4}\bm{I}_d)+u\big]}{n}\right)    \end{equation}
with probability at least $1-3e^{-u}$, where $r(\bm{\Sigma}+\frac{\lambda^2}{4}\bm{I}_d)=\frac{\tr(\bm{\Sigma})+{\lambda^2d}/{4}}{\|\bm{\Sigma}\|_{op}+{\lambda^2}/{4}}$ is the effective rank of $\bm{\Sigma}+\frac{\lambda^2}{4}\bm{I}_d$.
\end{theorem}

\begin{proof}
In order to invoke Theorem \ref{lem3} to bound $\|\bm{\widetilde{\Sigma}}_{\mb}-\bm{\Sigma}\|_{op}$, we use (\ref{2.1}) and write $$\bm{\widetilde{\Sigma}}_{\mb}-\bm{\Sigma}=\frac{1}{n}\sum_{i=1}^ n \bm{\dot{X}}_i\bm{\dot{X}}_i^\top -\Big(\bm{\Sigma}+\frac{\lambda^2}{4}\bm{I}_d\Big)= \frac{1}{n}\sum_{i=1}^n \bm{\dot{X}}_i\bm{\dot{X}}_i^\top- \mathbbm{E}\big(\bm{\dot{X}}_i\bm{\dot{X}}_i^\top\big).$$  We let $\bm{\xi}_i= \mathcal{Q}_\lambda(\bm{X}_i+\bm{\tau}_i)-\bm{X}_i$ be the quantization noise,   then Lemma \ref{lem1} gives $\|\bm{\xi}_i\|_{\psi_2}=O(\lambda)$.  Thus, for any $\bm{v}\in \mathbb{R}^{d}$ we have  \begin{equation}
    \|\langle\bm{\dot{X}}_i,\bm{v}\rangle\|_{\psi_2}\leq \|\langle\bm{X}_i,\bm{v}\rangle\|_{\psi_2}+\|\langle \bm{\xi}_i,\bm{v}\rangle\|_{\psi_2}\leq K\|\langle \bm{X}_i,\bm{v}\rangle\|_{\psi_2}+O(\lambda\|\bm{v}\|_2). \label{2723}
\end{equation} Note that (\ref{2.1}) yields $$\|\langle\bm{\dot{X}}_i,\bm{v}\rangle\|_{L^2}=\sqrt{\bm{v}^\top\mathbbm{E}[\bm{\dot{X}}_i\bm{\dot{X}}_i^\top]\bm{v}}=\Omega(\|\langle \bm{X}_i,\bm{v}\rangle\|_{L^2})+\Omega(\lambda\|\bm{v}\|_2),$$  which together with (\ref{2723}) implies $\|\langle \bm{\dot{X}}_i,\bm{v}\rangle\|_{\psi_2}=O(K+1)\|\langle \bm{\dot{X}}_i,\bm{v}\rangle\|_{L^2}$. Since in Definition \ref{def1} it must hold that $K=\Omega(1)$, $\bm{\dot{X}}_i$ is $O(K)$-sub-Gaussian. The result follows by invoking Theorem \ref{lem3}.
\end{proof}

\subsection{A New 2-Bit Covariance Estimator}

Having introduced the multi-bit estimator $\bm{\widetilde{\Sigma}}_{\mb}$, we propose a new 2-bit estimator  by properly restricting the  number of bits needed in the uniform quantizer. We begin with several useful observations. For the quantization of a scalar $a$, we first note  that the uniform quantizer $\mathcal{Q}_\lambda(\cdot)$ reduces to the 1-bit quantizer $\sign(\cdot)$ for signal $a$ bounded by the resolution $\lambda$, up to a scaling factor of $\frac{\lambda}{2}$: 
\begin{equation}\label{1brela1}
    \mathcal{Q}_\lambda(a)= \frac{\lambda}{2}\sign(a),\quad\text{when }|a|<\lambda.
\end{equation}
This remains true when the quantizer is associated with a uniform dither $\tau\sim \mathscr{U}[-\frac{\lambda}{2},\frac{\lambda}{2}]$, when the signal $a$ is sufficiently bounded away from $\lambda$: 
\begin{equation}\label{1brela2}
    \mathcal{Q}_\lambda(a+\tau)=\frac{\lambda}{2}\sign(a+\tau),\quad\text{when }|a|<\frac{\lambda}{2}.
\end{equation}
Nonetheless, such relation between $\mathcal{Q}_\lambda(\cdot)$ and $\sign(\cdot)$ is no longer valid if a triangular dither $\tau\sim \mathscr{U}[-\frac{\lambda}{2},\frac{\lambda}{2}]+\mathscr{U}[-\frac{\lambda}{2},\frac{\lambda}{2}]$ is adopted. Actually, as a triangular dither (see (\ref{triangulard})) takes value   on $[-\lambda,\lambda]$, for a non-constant signal $a$, the dithered signal $a+\tau$ falls on more than two bins of $\mathcal{Q}_\lambda(\cdot)$ with positive probability, making it impossible to reduce $\mathcal{Q}_\lambda(\cdot)$ to 1-bit quantization. Fortunately, under triangular dither $\tau\sim \mathscr{U}[-\frac{\lambda}{2},\frac{\lambda}{2}]+\mathscr{U}[-\frac{\lambda}{2},\frac{\lambda}{2}]$, we can still reduce $\mathcal{Q}_\lambda(\cdot)$  to a 2-bit quantizer for signal $a$ bounded by $\lambda$:
\begin{equation}\label{2.14}
\mathcal{Q}_\lambda(a+\tau)=\mathcal{Q}_{\lambda,\tb}(a+\tau),\quad\text{when }|a|<\lambda,
    \end{equation}
 where $\mathcal{Q}_{\lambda,\tb}(\cdot)$ is defined as\begin{equation}\label{2.15}
 \mathcal{Q}_{\lambda,\tb}(a):= -\frac{3\lambda}{2}\mathbbm{1}(a<-\lambda)-\frac{\lambda}{2}\mathbbm{1}(-\lambda\leq a<0)+\frac{\lambda}{2}\mathbbm{1}(0\leq a<\lambda)+\frac{3\lambda}{2}\mathbbm{1}(a\geq \lambda).
\end{equation}   
The relations (\ref{1brela1}), (\ref{1brela2}) and (\ref{2.14}) should be self-evident from the graphical illustration in Figure \ref{figadd}.
\begin{figure}[ht!]
    \centering
    \includegraphics[scale = 0.75]{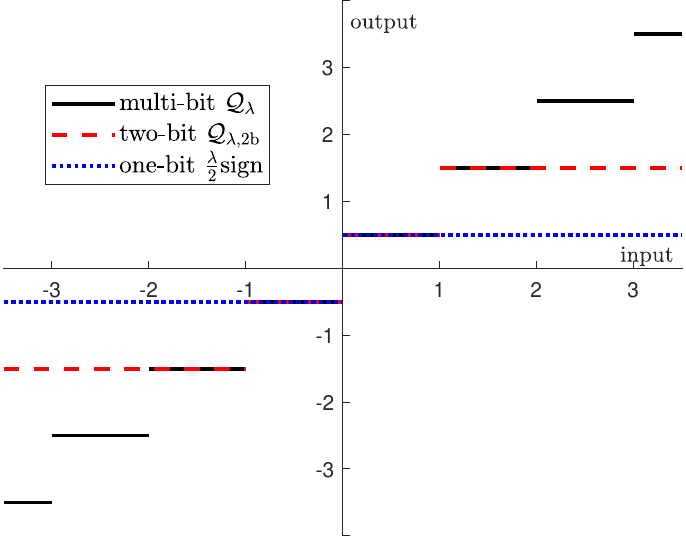}

    \caption{\footnotesize {\color{black}Graphical illustration of the multi-bit quantizer $\mathcal{Q}_{\lambda}(\cdot)$, the 2-bit quantizer $\mathcal{Q}_{\lambda,\tb}(\cdot)$ and the 1-bit quantizer $\frac{\lambda}{2}\sign(\cdot)$, under resolution $\lambda=1$.}} 
    \label{figadd}
\end{figure}

In essence, (\ref{2.14}) states that under triangular dither $\bm{\tau}_i$, $\mathcal{Q}_\lambda(\cdot+\bm{\tau}_i)$
and $\mathcal{Q}_{\lambda,\tb}(\cdot+\bm{\tau}_i)$ are close (or   the same) if the signal magnitude is dominated by (or exactly bounded by) $\lambda$. This inspires us to modify $\bm{\widetilde{\Sigma}}_{\mb}$ in (\ref{2.4}) and propose the following new non-adaptive 2-bit covariance matrix estimator $\bm{\widetilde{\Sigma}}_{\tb}$:
\begin{equation}\label{2.16}
    \bm{\widetilde{\Sigma}}_{\tb}= \frac{1}{n}\sum_{i=1}^n\bm{\dot{X}}_i\bm{\dot{X}}_i^\top - \frac{\lambda^2}{4}\bm{I}_d,\quad\text{with }\bm{\dot{X}}_i=\mathcal{Q}_{\lambda,\tb}(\bm{X}_i+\bm{\tau}_i),
\end{equation}
where $\{\bm{\tau}_i\}_{i=1}^n\sim \mathscr{U}[-\frac{\lambda}{2},\frac{\lambda}{2}]^d+\mathscr{U}[-\frac{\lambda}{2},\frac{\lambda}{2}]^d$.
\begin{theorem}\label{thm4}
    In the setting of Theorem \ref{thmmb}, we   let $\lambda^2=CK^2\|\bm{\Sigma}\|_\infty \log (nd)$ with some large enough $C$. Then with probability at least $1-10(nd)^{-10}-3\exp(-\frac{d}{2})$, the 2-bit estimator $\bm{\widetilde{\Sigma}}_{\tb}$ satisfies \begin{equation}
        \|\bm{\widetilde{\Sigma}}_{\tb}-\bm{\Sigma}\|_{op} \lesssim \sqrt{\frac{\hat{K}^2 d\|\bm{\Sigma}\|_\infty \|\bm{\Sigma}\|_{op}\log^2(nd)}{n}}+\frac{\hat{K}^2d\|\bm{\Sigma}\|_\infty \log (nd)}{n}.\label{34}
 \end{equation}
 Particularly, under the scaling of $n\gtrsim \frac{\hat{K}^2d\|\bm{\Sigma}\|_\infty}{\|\bm{\Sigma}\|_{op}}$, $$ \|\bm{\widetilde{\Sigma}}_{\tb}-\bm{\Sigma}\|_{op} \lesssim \sqrt{\frac{\hat{K}^2 d\|\bm{\Sigma}\|_\infty \|\bm{\Sigma}\|_{op}\log^2(nd)}{n}}$$ holds with high probability.
\end{theorem}
\begin{proof}    Our strategy is to show the two events \begin{gather*}
    \mathscr{E}_1=\Big\{\|\bm{\widetilde{\Sigma}}_{\mb}-\bm{\Sigma}\|_{op}\lesssim \sqrt{\frac{\hat{K}^2 d\|\bm{\Sigma}\|_\infty \|\bm{\Sigma}\|_{op}\log^2(nd)}{n}}+\frac{\hat{K}^2d\|\bm{\Sigma}\|_\infty \log (nd)}{n}\Big\}\\\text{and }~ \mathscr{E}_2= \big\{\bm{\widetilde{\Sigma}}_{\mb}=\bm{\widetilde{\Sigma}}_{\tb}\big\}
\end{gather*}      simultaneously hold with the promised probability.

    First, by invoking Theorem \ref{thmmb} with $$u =r\big(\bm{\Sigma}+\frac{\lambda^2}{4}\bm{I}_d\big)= \frac{\tr(\bm{\Sigma})+{\lambda^2d}/{4}}{\|\bm{\Sigma}\|_{op}+{\lambda^2}/{4}} \ge \frac{\lambda^2d/4}{2\max\{d\|\bm{\Sigma}\|_{\infty},\lambda^2/4\}} \ge \min\Big\{C_0K^2\log(nd),\frac{d}{2}\Big\}$$ for some large enough $C_0$,  we obtain that with probability at least $1-3(nd)^{-10}-3\exp(-\frac{d}{2})$, \begin{equation}
        \label{35}
        \|\bm{\widetilde{\Sigma}}_{\mb}-\bm{\Sigma}\|_{op} \lesssim \sqrt{\frac{\hat{K}^2d\|\bm{\Sigma}\|_\infty \|\bm{\Sigma}\|_{op}\log(nd)\big(1+\|\bm{\Sigma}\|_\infty\|\bm{\Sigma}\|_{op}^{-1}\log(nd)\big)}{n}}+\frac{\hat{K}^2d\|\bm{\Sigma}\|_\infty \log (nd)}{n}.
    \end{equation}
    By $\|\bm{\Sigma}\|_\infty\leq \|\bm{\Sigma}\|_{op}$, this implies $\mathscr{E}_1$.

    Second, we note that $\mathscr{E}_2$ holds if $\mathcal{Q}_{\lambda}(\bm{X}_i+\bm{\tau}_i)=\mathcal{Q}_{\lambda,\tb}(\bm{X}_i+\bm{\tau}_i) $ for any $i\in[n]$, and from  (\ref{2.14}) this can be implied by $\lambda \geq \max_{1\leq i\leq n}\|\bm{X}_i\|_\infty= \max_{i,j}|X_{ij}|$. As $\bm{X}_i$ is $K$-sub-Gaussian, for any $(i,j)$ we have $\|X_{ij}\|_{\psi_2}\leq K\sqrt{\Sigma_{jj}}\leq K\sqrt{\|\bm{\Sigma}\|_\infty}.$ Thus, for any $t\geq 0$ we have $$\mathbbm{P}\Big(|X_{ij}|\geq t\Big)\leq 2\exp\left(-\frac{C_1t^2}{K^2\|\bm{\Sigma}\|_\infty}\right).$$ Taking a union bound yields $$\mathbbm{P}\Big(\max_{i,j}|X_{ij}|\geq t\Big)\leq 2nd\exp\left(-\frac{C_1t^2}{K^2\|\bm{\Sigma}\|_\infty}\right).$$   Letting $t=C_2K\sqrt{\|\bm{\Sigma}\|_\infty\log(nd)}$ with large enough $C_2$, we obtain $\max_{i,j}|X_{ij}|=O(K\sqrt{\|\bm{\Sigma}\|_\infty\log(nd)})$ with probability at least $1-3(nd)^{-10}$. Under our choice $\lambda^2=CK^2\|\bm{\Sigma}\|_\infty\log(nd)$ with sufficiently large $C$, $\mathscr{E}_2$ thus holds with the promised probability. The proof is complete.
\end{proof}

Despite the distinct quantization processes and constructions, our new 2-bit estimator $\bm{\widetilde{\Sigma}}_{\tb}$ is indeed comparable to $\bm{\widehat{\Sigma}}_{na}$ in \cite{dirksen2022covariance} in the following senses:
\begin{itemize}
    \item[(i)] They both rely on two bits per entry; 
    \item[(ii)] They enjoy similar non-asymptotic operator norm error rate; c.f., (\ref{1.71}) and (\ref{34});
    \item[(iii)]   They both involve a tuning parameter of the dithering scale $\lambda$ which should be tuned according to $\|\bm{\Sigma}\|_\infty$ and $K$.
\end{itemize}
      Thus, like $\bm{\widehat{\Sigma}}_{na}$, our $\bm{\widetilde{\Sigma}}_{\tb}$ still suffers from the theoretical gap when $\tr(\bm{\Sigma})\ll d\|\bm{\Sigma}\|_\infty$ and the tuning issue. This can be numerically corroborated by comparing (a) and (c) of Figure \ref{fig1}, where we   include $\bm{\widetilde{\Sigma}}_{\tb}$ with $\lambda=C\sqrt{\log n}$ for comparison, and note that in (c) the optimal C for $\widetilde{\bm{\Sigma}}_{\tb}$ shifts to the right and exhibits larger performance gap compared to sample covariance. On the other hand, in all three numerical examples $\bm{\widetilde{\Sigma}}_{\tb}$ notably outperforms $\bm{\widehat{\Sigma}}_{na}$, which suggests that our new 2-bit estimator is preferable for Gaussian samples. 

 
\section{A Parameter-Free Estimator with Improved Error Rate}\label{sec3}
As an intuitive explanation on the gap between $\bm{\widehat{\Sigma}}_{na},\bm{\widetilde{\Sigma}}_{\tb}$ and the full-data-based sample covariance, it was written on \cite[P. 3544]{dirksen2022covariance} that ``in order to accurately estimate all diagonal entries of $\bm{\Sigma}$, the (maximal) dithering level $\lambda$ needs to be on the scale $\|\bm{\Sigma}\|_\infty$; if $\tr(\bm{\Sigma})\ll d\|\bm{\Sigma}\|_\infty$, then most of the diagonal entries are much smaller than $\|\bm{\Sigma}\|_\infty$, and hence $\lambda$ is on a sub-optimal scale for these entries." In a nutshell, the common dithering scale $\lambda$ should be large enough to accommodate the large entries and hence could be highly sub-optimal for other small entries. While \cite{dirksen2023tuning} deployed data-driving dithering, the dithering scale remains the same across entries of a specific sample, thus their estimator suffers from the same gap. Naturally, a possible strategy to overcome the limitation is to adopt dithering  scales adaptive to different entries of $\bm{X}_i$.

Along this idea, we propose in this section our main estimator that simultaneously resolves the issues of sub-optimal rate and tuning parameter. In notation, we let $\lambda_j$ be the dithering scale for the $j$-th entries $\{X_{ij}:i\in[n]\}$ and will quantize $X_{ij}$ to $\dot{X}_{ij}=\mathcal{Q}_{\lambda_j,\tb}(X_{ij}+\lambda_j\tau_{ij})$ with (normalised) triangular dither $\tau_{ij}\sim \mathscr{U}[-\frac{1}{2},\frac{1}{2}]+\mathscr{U}[-\frac{1}{2},\frac{1}{2}]$. We comment that an appropriate $\lambda_j$ should represent a reasonable trade-off between bias and variance. On one hand,     (\ref{2.14}) indicates that using $\lambda_j\ll |X_{ij}|$ induces large bias between $\mathcal{Q}_{\lambda_j,\tb}(\cdot)$ and $\mathcal{Q}_{\lambda_j}(\cdot)$, and as a consequence   our 2-bit estimator could deviate largely from its unbiased multi-bit counterpart. On the other hand,  $\lambda_j\gg 
|X_{ij}|$ is also sub-optimal and 
 results in large variance (or slow concentration) of the quantized samples. It seems that we shall choose $\lambda_j$ at a scale comparable to $\{|X_{ij}|:i\in[n]\}$. 
Fortunately,  as we will see, the  tightest unbiased choices
\begin{equation}\label{3.1}
    \lambda_{j,\min} = \max_{1\leq i\leq n}|X_{ij}|,~j=1,...,d
\end{equation} 
are sufficient for closing the theoretical gap. Additionally, such $\lambda_{j,\min}$ is   determined by the given data, so (\ref{3.1}) also completely addresses the tuning issue.

Now we are ready to precisely propose our estimator. Let the triangular dithers $\{\bm{\tau}_i\}_{i=1}^n\stackrel{iid}{\sim}\mathscr{U}[-\frac{1}{2},\frac{1}{2}]^d+ \mathscr{U}[-\frac{1}{2},\frac{1}{2}]^d$ be independent of $\{\bm{X}_i\}_{i=1}^n$, $$\bm{\Lambda}_{\min}=\diag(\lambda_{1,\min},...,\lambda_{d,\min})$$ where $\lambda_{j,\min}$ is given in (\ref{3.1}), then we define the quantizer $$\bm{\mathcal{Q}}_{\bm{\Lambda}_{\min},\tb}(\cdot):=(\mathcal{Q}_{\lambda_{1,\min},\tb}(\cdot),...,\mathcal{Q}_{\lambda_{d,\min},\tb}(\cdot))^\top$$ which element-wisely quantizes $\bm{X}_i$ to \begin{equation}\label{3.2}
    \begin{aligned}\bm{\dot{X}}_i &= \bm{\mathcal{Q}}_{\bm{\Lambda}_{\min},2b}(\bm{X}_i+\bm{\Lambda}_{\min}\bm{\tau}_i)\\
    :&=\Big(\mathcal{Q}_{\lambda_{1,\min},\tb}(X_{i1}+\lambda_{1,\min}\tau_{i1}),...,\mathcal{Q}_{\lambda_{d,\min},\tb}(X_{id}+\lambda_{d,\min}\tau_{id})\Big)^\top\\
    &\stackrel{(i)}{=}\Big(\mathcal{Q}_{\lambda_{1,\min}}(X_{i1}+\lambda_{1,\min}\tau_{i1}),...,\mathcal{Q}_{\lambda_{d,\min}}(X_{id}+\lambda_{d,\min}\tau_{id})\Big)^\top;
    \end{aligned}
\end{equation}
note that $(i)$ follows from (\ref{2.14}) and (\ref{3.1}). Then, as can be shown by some intermediate results in the proof of Theorem \ref{maintheorem}, we have $\mathbbm{E}(\bm{\dot{X}}_i\bm{\dot{X}}_i^\top)=\bm{\Sigma}+\frac{1}{4}\bm{\Lambda}_{\min}^2$. Therefore, 
we propose our main  estimator as \begin{equation}\label{pfesti}
    \bm{\widetilde{\Sigma}}=\frac{1}{n}\sum_{i=1}^n\bm{\dot{X}}_i\bm{\dot{X}}_i^\top-\frac{1}{4}\bm{\Lambda}_{\min}^2,\quad\text{with }\bm{\dot{X}}_i  = \bm{\mathcal{Q}}_{\bm{\Lambda},\tb}(\bm{X}_i+\bm{\Lambda}_{\min}\bm{\tau}_i),
\end{equation}
where $\{\bm{\tau}_i\}_{i=1}^n\sim \mathscr{U}[-\frac{1}{2},\frac{1}{2}]^d+ \mathscr{U}[-\frac{1}{2},\frac{1}{2}]^d.$

Our primary theoretical result is given as below. When omitting dependence on $K$ and a small number of logarithmic factors, there is no difference between our (\ref{desired}) for $\bm{\widetilde{\Sigma}}$ and (\ref{1.444}) for sample covariance.  

\begin{theorem}\label{maintheorem}
    Suppose that $\bm{X}_1,...,\bm{X}_n$ are i.i.d. copies of the zero-mean $K$-sub-Gaussian random vector $\bm{X}\in \mathbb{R}^d$. For estimating $\bm{\Sigma}=\mathbbm{E}(\bm{X}\bm{X}^\top)$ with effective rank $r(\bm{\Sigma})=\frac{\tr(\bm{\Sigma})}{\|\bm{\Sigma}\|_{op}}$, under the scaling of $n=\Omega(\hat{K}^2\log(nd))$, 
     the proposed estimator $\bm{\widetilde{\Sigma}}$ satisfies 
    \begin{equation} 
         \|\bm{\widetilde{\Sigma}}-\bm{\Sigma}\|_{op} \lesssim K^2 \|\bm{\Sigma}\|_{op}\left(\sqrt{\frac{r(\bm{\Sigma})\log^2 (nd)}{n}}+ \frac{r(\bm{\Sigma})\log^2 (nd)}{n}\right)\label{desired}
    \end{equation}
    with probability at least $1-10(nd)^{-10}-10e^{-10r(\bm{\Sigma})}$. 
    Particularly,     if further assuming the scaling of $n\gtrsim r(\bm{\Sigma})\log^2(nd)$, with high probability we have \begin{equation}\label{desired2}
         \|\bm{\widetilde{\Sigma}}-\bm{\Sigma}\|_{op} \lesssim K^2 \sqrt{\frac{\mathrm{Tr}(\bm{\Sigma}) \|\bm{\Sigma}\|_{op}\log ^2(nd)}{n}}.
    \end{equation}
\end{theorem}
\begin{proof}[A Proof Sketch for Theorem \ref{maintheorem}]
      The key challenge in the proof is that the actual dithering vectors $\bm{\Lambda}_{\min}\bm{\tau}_i$ depend on $\{\bm{X}_i\}_{i=1}^n$ through the quantities $\{\lambda_{j,\min}\}_{j=1}^d$; thus, when simultaneously handling the randomness of $\{\bm{X}_i,\bm{\tau}_i\}_{i=1}^n$, $\{\bm{\dot{X}}_i\bm{\dot{X}}_i^\top\}_{i=1}^n$ are no longer i.i.d. matrices but    correlated in a rather delicate manner,   making the simple arguments in Theorem \ref{thmmb} invalid. To get started, we must decompose the quantized sample into  $\bm{\dot{X}}_i=\bm{X}_i+\bm{\xi}_i$, and this allows us to divide the estimation error into three pieces as in (\ref{4.44}). Recall that the properties of the quantization noise $\bm{\xi}_i$ are developed in Lemma \ref{lem1}.  To bound each piece separately, we often need to first deal with the randomness of $\{\bm{\tau}_i\}_{i=1}^n$ by conditioning on $\{\bm{X}_i\}_{i=1}^n$ to render ``independence", and then deal with $\{\bm{X}_i\}_{i=1}^n$ that are again independent. In addition, we   need to carefully take the technical tools of Theorem \ref{lem3} or matrix Bernstein's inequality (Lemma \ref{lem2}) to allow for a sharp rate depending on the effective rank $r(\bm{\Sigma})$ instead of $d$ (while some other approaches like covering argument may not suffice for this purpose). 
    The complete proof is provided in Section \ref{sec4}. 
\end{proof}
\begin{rem}[Bias-Variance Tradeoff]\label{shrink} Recall that in $\bm{\widetilde{\Sigma}}$, our choice of dithering scales in (\ref{3.1}) is unbiased and   do not balance the bias and variance. Though this is sufficient for closing the theoretical error rate gap, numerically it is oftentimes preferable to slightly lower the dithering scales to allow for a better bias-and-variance trade-off. To this end, we propose to ``shrink'' the dithering scale by  a global factor $s\in(0,1)$ and 
use the dithering scales
$$s\bm{\Lambda}_{\min}={\rm diag}(s\lambda_{1,\min},...,s\lambda_{d,\min}).$$
This leads to the estimator \begin{equation}\label{shrinkageesti}
 \bm{\widetilde{\Sigma}}(s)=\frac{1}{n}\sum_{i=1}^n\bm{\dot{X}}_i\bm{\dot{X}}_i^\top-\frac{1}{4}(s\bm{\Lambda}_{\min})^2,\quad\text{with }    \bm{\dot{X}}_i=\bm{\mathcal{Q}}_{s\bm{\Lambda}_{\min},\tb}(\bm{X}_i+s\bm{\Lambda}_{\min}\bm{\tau}_i).
\end{equation}     Note that $\bm{\widetilde{\Sigma}}(1)=\bm{\widetilde{\Sigma}}$. {\color{black}  In practice, we shall advocate ``safely'' using $s$ slightly below $1$ (such as $s\in[0.8,0.9]$). While as we shall see from numerical simulation, $s=0.5$ is a generally good choice under Gaussian samples and further improves on $s=0.7,~0.9$.} 
\end{rem}
\begin{rem}
    [Online Setting] \label{online}{\color{black} One caveat of our estimator $\widetilde{\bm{\Sigma}}$ is that a full pass over all unquantized data is needed to set the dithering scales, thus precluding its applications in the online setting where the samples arrive at a sensor sequentially. We shall first emphasize the effectiveness of $\widetilde{\bm{\Sigma}}$ in the already ubiquitous offline settings where one can  find $(\lambda_{j,\min})_{j=1}^d$ before quantization, in that it allows for near-optimal covariance estimation from 2-bit data;    
    for instance, working with quantized data can reduce communication cost in many   distributed machine learning systems (e.g., see  \cite{chen2022quantizing}). Moreover, our main idea of using adaptive dithering scales for different coordinates goes beyond the specific estimator $\widetilde{\bm{\Sigma}}$, and we will adapt our method to the online setting in Section \ref{sec6} by proposing an estimator that enjoys the same error rate and involves   a  tuning parameter depending only on $K$ (thus requiring the same level of tuning as $\widetilde{\bm{\Sigma}}_a$ from \cite{dirksen2023tuning}). }   
\end{rem}

\subsection{Numerical Simulations}
In our simulations, we compare the operator norm errors of $\bm{\widehat{\Sigma}}_{na}$ in (\ref{aosesti}), our new 2-bit estimator $\bm{\widetilde{\Sigma}}_{\tb}$ in (\ref{2.16}), the sample covariance $\bm{\widehat{\Sigma}}$ based on full data, and our parameter-free estimator $\bm{\widetilde{\Sigma}}$ in (\ref{pfesti}). We also consider the estimators $\widetilde{\bm{\Sigma}}(s)$ with $s\in\{0.5,0.7,0.9\}$ using the reduced dithering scales $\{s\lambda_{j,\min}\}_{j=1}^d$; see
 Remark \ref{shrink}. 

 We shall pause to state the tuning parameters for the above estimators. The non-adaptive estimators $\bm{\widehat{\Sigma}}_{na},\bm{\widetilde{\Sigma}}_{\tb}$ require a tuning parameter $\lambda=C\sqrt{\log n}$ for some $C$ depending on $(K,\|\bm{\Sigma}\|_\infty)$, and we provide the (near) optimal $C$ numerically found in the case of $(n,d)=(500,10)$,\footnote{For instance, from Figure \ref{fig1}(a), under $\bm{\Sigma}(1,0.2,1)$ we take $C=0.7$ for $\bm{\widehat{\Sigma}}_{na}$, $C=0.65$ for $\bm{\widetilde{\Sigma}}_{\tb}$.} but note that   it is hard to tune $C$ to such extent in practice. In contrast, our key estimator $\bm{\widetilde{\Sigma}}$ does not require any parameter, and its more ``balanced"
version  $\bm{\widetilde{\Sigma}}(s)$ is also user-friendly in that it only involves the parameter $s$ which can be manually set  to be moderately less than $1$ (like $s=0.9,~0.7,~0.5$ tested here). We will use $\bm{X}_i\sim \mathcal{N}(0,\bm{\Sigma})$ and average the estimation errors over 50 repetitions.   

\begin{figure}[ht!]
    \centering
    \includegraphics[scale = 0.52]{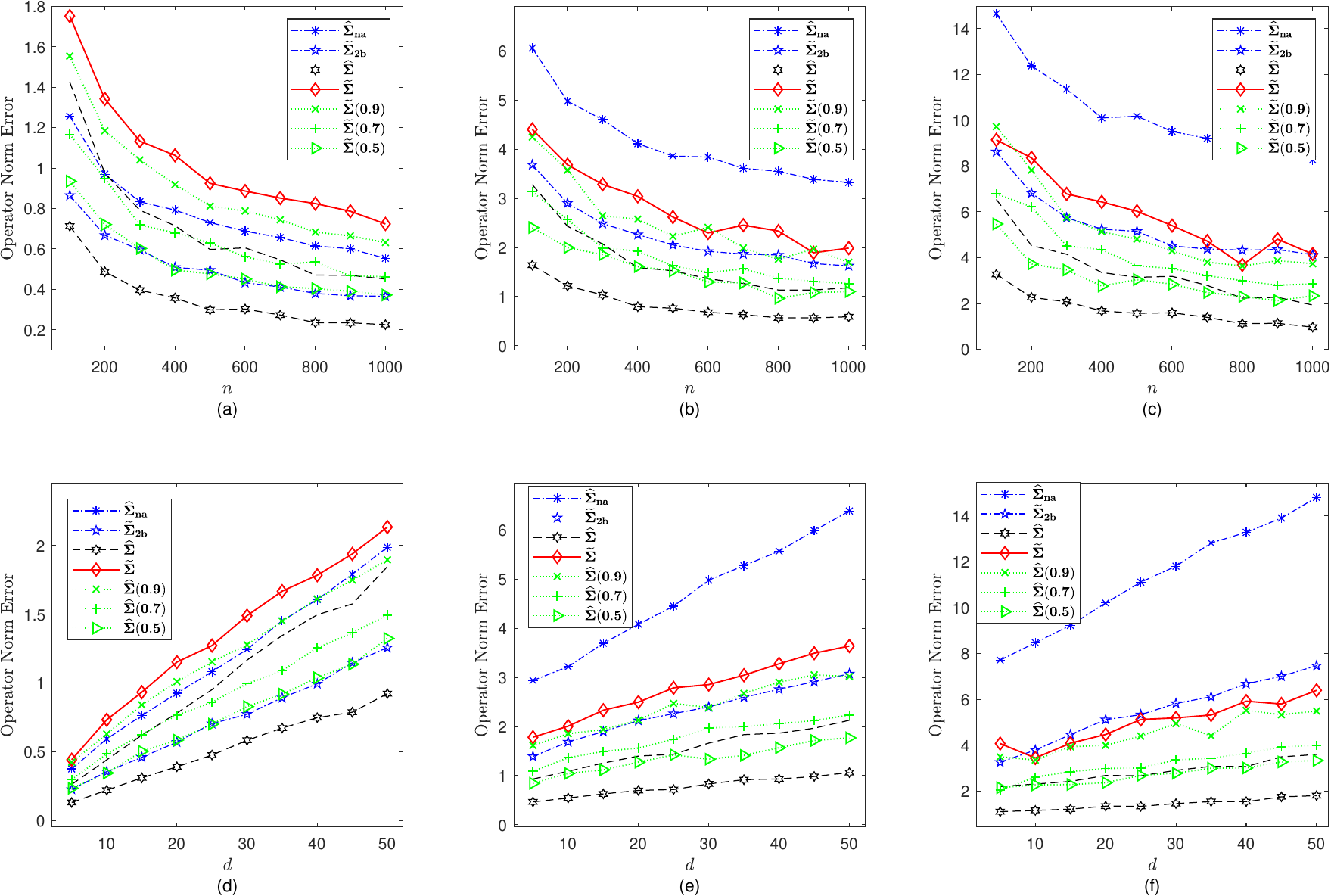}
    \caption{\footnotesize The curves of ``operator norm error v.s. $n$" and ``operator norm error v.s. $d$". We simulate $\bm{\Sigma} = \bm{\Sigma}(1,0.2,1)$ in (a) an (d), $\bm{\Sigma}=\bm{\Sigma}(1,0.2,10)$ in (b) and (e), and $\bm{\Sigma}=\bm{\Sigma}(1,0.2,25)$ in (c) and (f).} 
    \label{fig2}
\end{figure}

We first fix $d=10$ and test $n = 100:100:1000$, then report the results for $\bm{\Sigma}=\bm{\Sigma}(1,0.2,1)$, $\bm{\Sigma}(1,0.2,10)$ and $\bm{\Sigma}(1,0.2,25)$  in Figure \ref{fig2}(a)-(c), respectively.

For $\bm{\Sigma}=\bm{\Sigma}(1,0.2,1)$ that represents the case of $\tr(\bm{\Sigma}) =d\|\bm{\Sigma}\|_\infty$, in which 
the theoretical guarantees for 
$(\bm{\widehat{\Sigma}}_{na},\bm{\widetilde{\Sigma}}_{\tb},\bm{\widetilde{\Sigma}})$     almost coincide,  
$\bm{\widehat{\Sigma}}_{na}$ and $\bm{\widetilde{\Sigma}}_{\tb}$ with optimal $C$ achieve smaller errors than $\bm{\widetilde{\bm{\Sigma}}}$. While as expected, the dithering scales (\ref{3.1}) may be (slightly) numerically sub-optimal and proper shrinkage could further lower the errors of $\bm{\widetilde{\Sigma}}$. More specifically, $\bm{\widetilde{\Sigma}}(s)$ provides lower errors when decreasing s from 1 to 0.5, and  $\bm{\widetilde{\Sigma}}(0.5)$ performs comparably to $\bm{\widetilde{\Sigma}}_{\tb}$ with the best $C$.

To corroborate our   theoretical achievement, of more interest are the cases of $\bm{\Sigma}=\bm{\Sigma}(1,0.2,10)$ and $\bm{\Sigma}(1,0.2,25)$ where $\tr(\bm{\Sigma}) \ll d\|\bm{\Sigma}\|_\infty$, and hence   (\ref{desired}) for $\bm{\widetilde{\Sigma}}$ significantly improves on (\ref{1.71}) and (\ref{34}) for $\bm{\widehat{\Sigma}}_{na}$ and $\bm{\widetilde{\Sigma}}_{\tb}$. From Figure \ref{fig2}(b)-(c), $\bm{\widetilde{\Sigma}}$ already outperforms $\bm{\widehat{\Sigma}}_{na}$ and performs only slightly worse than $\bm{\widetilde{\Sigma}}_{\tb}$, although the latter two involve unrealistic tuning. By shrinking the dithering scales, $\bm{\widetilde{\Sigma}}(0.7)$ and $\bm{\widetilde{\Sigma}}(0.5)$ provide operator norm errors notably lower than  $\bm{\widetilde{\Sigma}}_{\tb}$, thus validating the benefit of  data-driven dithering with scales varying across coordinates.

To compare with sample covariance based on full data, we plot the {\it unlabeled black dashed curves that double the sample covariance's errors.} Remarkably, the curves of $\bm{\widetilde{\Sigma}}(0.5)$ are almost always below the black dashed ones. This indicates  that over Gaussian samples,  by properly shrinking the dithering scales in (\ref{3.1}), the proposed estimator oftentimes achieves operator norm  errors no greater than twice of the errors of sample covariance.

Furthermore, we also test a fixed $n=500$ and the increasing ambient dimension $d=5:5:50$ under the above three covariance matrices. The results with similar implications are provided in Figure \ref{fig2}(d)-(f).

\section{The Proof of Main Theorem}\label{sec4}
The most important ingredient in our analysis is the following matrix Bernstein's inequality.

\begin{lem}\label{lem2}
    {\rm(Matrix Bernstein's Inequality, e.g., \cite[Thm. 6.1.1]{tropp2015introduction}).} Let $\bm{S}_1,\bm{S}_2,...,\bm{S}_n$ be independent, zero-mean, $d_1\times d_2$   random matrices such that $\|\bm{S}_i\|_{op}\leq L$ almost surely for all $i$. Let $$\sigma^2:=\max\left\{\Big\|\sum_{i=1}^n \mathbbm{E}(\bm{S}_i\bm{S}_i^\top)\Big\|_{op}, \Big\|\sum_{i=1}^n \mathbbm{E}(\bm{S}_i^\top\bm{S}_i)\Big\|_{op}\right\},$$ then for some absolute constant $c$, for any $t\geq 0$ we have 
    \begin{equation}\nonumber
\mathbbm{P}\left(\Big\|\sum_{i=1}^n\bm{S}_i\Big\|_{op}\geq t\right)\leq (d_1+d_2)\exp\left(-c\min\Big\{\frac{t^2}{\sigma^2},\frac{t}{L}\Big\}\right).
    \end{equation}
\end{lem}

We will also use the concentration of the norm of a sub-Gaussian vector. 
\begin{lem}\label{lem4}
    {\rm (Adapted from \cite[Thm. 4.1]{jeong2022sub})} Let $\bm{a}=(a_1,...,a_n)\in \mathbb{R}^n$ be a random vector with independent sub-Gaussian coordinates $a_i$ that satisfy $\mathbbm{E}(a_i^2)=1$. Let $A:=\max_{1\leq i\leq n}\|a_i\|_{\psi_2}$ and $\hat{A}=A\sqrt{\log A}$, then for some absolute constant $C$ we have $\big\|\|\bm{a}\|_2-\sqrt{n}\big\|_{\psi_2}\leq C\hat{A},$  which further implies the following for some $C_1$: \begin{equation}\nonumber
\mathbbm{P}\Big(\big|\|\bm{a}\|_2-\sqrt{n}\big|\geq t\Big) \leq 2\exp \left(-\frac{C_1t^2}{\hat{A}^2}\right),\quad\forall~t\geq 0.  
\end{equation}\label{lemvectornorm}
\end{lem}

We are ready to present the proof of our main theorem.
\begin{proof}[The proof of Theorem \ref{maintheorem}]

{\color{black}In this proof, we simply write $\bm{\Lambda}_{\min}=\diag(\lambda_{1,\min},...,\lambda_{d,\min})$ as $\bm{\Lambda}$.}
    To get started, we prove a useful observation: \begin{equation}
        \label{4.33}\bm{\dot{X}}_i=\bm{\mathcal{Q}}_{\bm{\Lambda},2b}(\bm{X}_i+\bm{\Lambda\tau_i})=\bm{\Lambda}\mathcal{Q}_1(\bm{\Lambda}^{-1}\bm{X}_i+\bm{\tau}_i).
    \end{equation} Note that the first equality  is from  (\ref{3.2}), and now let us prove the second equality. To this end, we start from the last line of (\ref{3.2}) and calculate the $j$-th entry of $\bm{\dot{X}}_i$ as\begin{equation}
      \begin{aligned}  \nonumber{\dot{X}}_{ij}&=\mathcal{Q}_{\lambda_{j,\min}}(X_{ij}+\lambda_{j ,\min} \tau_{ij})\\& = \lambda_{j,\min}\left(\left\lfloor \frac{X_{ij}+\lambda_{j,\min}\tau_{ij}}{\lambda_{j,\min}}\right\rfloor+\frac{1}{2}\right)\\
      & = \lambda_{j,\min}\left(\left\lfloor \frac{X_{ij}}{\lambda_{j,\min}}+\tau_{ij}\right\rfloor+\frac{1}{2}\right)= \lambda_{j,\min}\cdot\mathcal{Q}_1\left(\frac{X_{ij}}{\lambda_{j,\min}}+\tau_{ij}\right),
      \end{aligned}
    \end{equation} 
    then (\ref{4.33}) follows by noting that the last quantity in the above equation is just the $j$-th entry of $\bm{\Lambda}\mathcal{Q}_1(\bm{\Lambda}^{-1}\bm{X}_i+\bm{\tau}_i)$. To proceed,
    we define the quantization noise as  \begin{equation}
        \label{91quannoise}
\bm{\xi}_i=\mathcal{Q}_1(\bm{\Lambda}^{-1}\bm{X}_i+\bm{\tau}_i)-\bm{\Lambda}^{-1}\bm{X}_i
    \end{equation} and let $\xi_{ij}$ be its $j$-th entry. This  provides $\bm{\dot{X}}_i=\bm{\Lambda}\mathcal{Q}_1(\bm{\Lambda}^{-1}\bm{X}_i+\bm{\tau}_i)=\bm{X}_i+\bm{\Lambda}\bm{\xi}_i$, which further leads to \begin{equation}
        \begin{aligned}\label{4.44}
            \bm{\widetilde{\Sigma}}-\bm{\Sigma} &=\frac{1}{n}\sum_{i=1}^n\big(\bm{X}_i+\bm{\Lambda\xi}_i\big)\big(\bm{X}_i+\bm{\Lambda\xi}_i\big)^\top - \frac{1}{4}\bm{\Lambda}^2- \bm{\Sigma}\\
            &=\Big(\frac{1}{n}\sum_{i=1}^n\bm{X}_i\bm{X}_i^\top-\bm{\Sigma}\Big)+\Big(\frac{1}{n}\sum_{i=1}^n \bm{\Lambda}\bm{\xi}_i\bm{\xi}_i^\top\bm{\Lambda}^\top-\frac{1}{4}\bm{\Lambda}^2\Big)\\& ~~~+ \frac{1}{n}\sum_{i=1}^n\Big(\bm{X}_i\bm{\xi}_i^\top\bm{\Lambda}+\bm{\Lambda}\bm{\xi}_i\bm{X}_i^\top\Big):=I_1+I_2+I_3.
        \end{aligned}
    \end{equation}
    Thus, by triangular inequality we have \begin{equation}\label{finalback}
        \|\bm{\widetilde{\Sigma}}-\bm{\Sigma}\|_{op}\leq \|I_1\|_{op}+\|I_2\|_{op}+\|I_3\|_{op}.
    \end{equation}  

    \subsubsection*{Step 1: Bounding $\|I_1\|_{op}$}

\noindent
     Theorem \ref{lem3}  gives that for any $u_1\geq 0$, with probability at least $1-3e^{-u_1}$ it holds that 
     \begin{equation}\nonumber
         \|I_1\|_{op}\leq C\|\bm{\Sigma}\|_{op}\left(\sqrt{\frac{\hat{K}^2[r(\bm{\Sigma})+u_1]}{n}}+\frac{\hat{K}^2[r(\bm{\Sigma})+u_1]}{n}\right).
     \end{equation}
Setting $u_1=10r(\bm{\Sigma})$,   we obtain that  \begin{equation}
    \label{I1bound}
    \|I_1\|_{op} \lesssim \|\bm{\Sigma}\|_{op} \left(\sqrt{\frac{\hat{K}^2r(\bm{\Sigma})}{n}}+\frac{\hat{K}^2r(\bm{\Sigma})}{n}\right)
\end{equation}
holds with probability at least $1-3e^{-10r(\bm{\Sigma})}$.
    
     \subsubsection*{Step 2: Bounding $\|I_2\|_{op}$}
 
     {\it 2.1) Conditioning on $\{\bm{X_i}\}_{i=1}^n$ and using the randomness of $\{\bm{\tau}_i\}_{i=1}^n$}

     \noindent
     Since $\bm{\xi}_i=\mathcal{Q}_1(\bm{\Lambda}^{-1}\bm{X}_i+\bm{\tau}_i)-\bm{\Lambda}^{-1}\bm{X}_i$, conditioning on $\{\bm{X}_i\}_{i=1}^n$, $\bm{\xi}_i$'s are independent of each other, and Lemma \ref{lem1} gives $\mathbbm{E}(\bm{\xi}_i\bm{\xi}_i^\top)=\frac{1}{4}\bm{I}_d$.   Hence, by letting $\bm{\eta}_i=\bm{\Lambda\xi}_i$,  $I_2$ can be written as 
         $$I_2=\frac{1}{n}\sum_{i=1}^n(\bm{\Lambda\xi}_i)(\bm{\Lambda\xi}_i)^\top-\mathbbm{E}[(\bm{\Lambda\xi}_i)(\bm{\Lambda\xi}_i)^\top]=\frac{1}{n}\sum_{i=1}^n\bm{\eta}_i\bm{\eta}_i^\top-\mathbbm{E}[\bm{\eta}_i\bm{\eta}_i^\top].$$ By Lemma \ref{lem1} we have $\|\bm{\xi}_i\|_{\psi_2}=O(\lambda)$, and thus for any $\bm{v}\in \mathbb{R}^d$ it holds that $$\|\langle\bm{\eta}_i,\bm{v}\rangle\|_{\psi_2}=\|\langle\bm{\xi}_i,\bm{\Lambda v}\rangle\|_{\psi_2}=O(\|\bm{\Lambda v}\|_{2}).$$ Combining with $\|\langle\bm{\eta}_i,\bm{v}\rangle\|_{L^2}^2=\|\langle \bm{\xi}_i,\bm{\Lambda v}\rangle\|_{L^2}^2=\frac{1}{4}\|\bm{\Lambda v}\|_{2}^2$, we know that $\bm{\eta}_i$'s are $O(1)$-sub-Gaussian under Definition \ref{def1}. Thus, Theorem \ref{lem3} implies that with probability at least $1-3e^{-u_2}$, it holds that \begin{equation}\nonumber
         \|I_2\|_{op}\lesssim \sqrt{\frac{\tr(\bm{\Lambda}^2)\|\bm{\Lambda}\|_{op}^2+\|\bm{\Lambda}\|_{op}^4u_2}{n}}+ \frac{\tr(\bm{\Lambda}^2)+\|\bm{\Lambda}^2\|_{op}u_2}{n}.
     \end{equation}
     Setting $u_2=10r(\bm{\Sigma})$, we obtain that \begin{equation}\label{3.16}
          \|I_2\|_{op}\lesssim \sqrt{\frac{\tr(\bm{\Lambda}^2)\|\bm{\Lambda}\|_{op}^2+\|\bm{\Lambda}\|_{op}^4r(\bm{\Sigma})}{n}}+ \frac{\tr(\bm{\Lambda}^2)+\|\bm{\Lambda}^2\|_{op}r(\bm{\Sigma})}{n}
     \end{equation} 
     holds with probability at least $1-3e^{-10r(\bm{\Sigma})}$.

\noindent{\it 2.2) Dealing with the randomness of $\{\bm{X}_i\}_{i=1}^n$}

\noindent    
Recall that $\bm{\Lambda}=\diag(\lambda_{1,\min},...,\lambda_{d,\min})$ where $\lambda_{j,\min}=\max_{1\leq i\leq n}|X_{ij}|$. Since $\|X_{ij}\|_{\psi_2}\leq K\sqrt{\mathbbm{E}|X_{ij}|^2}=K\sqrt{\Sigma_{jj}}$, for any $u_3 \geq 0$ we have $$\mathbbm{P}(|X_{ij}|\geq u_3)\leq 2\exp\left(-\frac{C_4u_3^2}{K^2\Sigma_{jj}}\right),$$ then a union bound over $i\in [n]$ gives 
\begin{equation}\nonumber
    \mathbbm{P}\big(\lambda_{j,\min}\geq u_3\big)\leq 2n\exp\left(-\frac{C_4u_3^2}{K^2\Sigma_{jj}}\right).
\end{equation}
We set $u_3=C_5K\sqrt{\Sigma_{jj}\log(nd)}$ with sufficiently large $C_5$ to obtain that, with probability at least $1-2(nd)^{-11}$ we have $\lambda_{j,\min}\leq C_6K\sqrt{\Sigma_{jj}\log (nd)}$. Further, a union bound over $j\in [d]$ yields that \begin{equation}\label{3.18}
    \mathbbm{P}\Big(\lambda_{j,\min}\leq C_6K\sqrt{\Sigma_{jj}\log (nd)},~\forall j\in [d]\Big) \geq 1-2(nd)^{-10}. 
\end{equation}  
On this high-probability event, it immediately follows that \begin{gather}\label{3.19}
    \|\bm{\Lambda}\|_{op}=\max_{1\leq j\leq d}\lambda_{j,\min}\leq C_6K \max_{1\leq j\leq d}\sqrt{\Sigma_{jj}\log (nd)}=C_6K\sqrt{\|\bm{\Sigma}\|_\infty \log (nd)},\\\label{3.20}
    \tr(\bm{\Lambda}^2)=\sum_{j=1}^d\lambda_{j,\min}^2\leq C_6^2K^2 \sum_{j=1}^d \Sigma_{jj}\log (nd)=C_6^2K^2\tr(\bm{\Sigma})\log (nd).
\end{gather}
Substituting (\ref{3.19}), (\ref{3.20}) into (\ref{3.16}) and using $r(\bm{\Sigma})\|\bm{\Sigma}\|_\infty \le \tr(\bm{\Sigma})$ yield that, with the promised probability, \begin{equation}\label{I2bound}
    \|I_2\|_{op}\leq C_7K^2 \log(nd) \left(\sqrt{\frac{\tr(\bm{\Sigma})\|\bm{\Sigma}\|_\infty}{n}}+\frac{\tr(\bm{\Sigma})}{n}\right).
\end{equation}

      \subsubsection*{Step 3: Bounding $\|I_3\|_{op}$}

It suffices to bound $\big\|\frac{1}{n}\sum_{i=1}^n\bm{X}_i\bm{\xi}_i^\top\bm{\Lambda}\big\|_{op}$
since $\|I_3\|_{op}\leq 2\big\|\frac{1}{n}\sum_{i=1}^n \bm{X}_i\bm{\xi}_i^\top\bm{\Lambda}\big\|_{op}$. We utilize matrix Bernstein's inequality  for this purpose.

\noindent{\it 3.1) Conditioning on $\{\bm{X}_i\}_{i=1}^n$ and using the randomness of $\{\bm{\tau}_i\}_{i=1}^n$}

\noindent
Given $\{\bm{X_i}:i\in [n]\}$, $\bm{\xi}_i$'s are independent, and   $\mathbbm{E}(\bm{\xi}_i)=0$ implies $\mathbbm{E}(\bm{X}_i\bm{\xi}_i^\top\bm{\Lambda})=0$.  Our goal is to bound the operator norm of the following  {\color{black}sum of independent zero-mean matrices:}\begin{equation}
    \Big\|\sum_{i=1}^n \bm{W}_i\Big\|_{op},\quad\text{where }\bm{W}_i:=\frac{1}{n}\bm{X}_i\bm{\xi}_i^\top\bm{\Lambda}.
\end{equation}
Since $\|\bm{\xi}_i\|_\infty \leq \frac{3}{2}$, we have\begin{equation}
    \label{Lamxi}
    \|\bm{\Lambda\xi}_i\|_2^2= \sum_{j=1}^d \lambda_{j,\min}^2 \xi_{ij}^2 \leq \frac{9}{4}\sum_{j=1}^d \lambda_{j,\min}^2  = \frac{9}{4}\tr(\bm{\Lambda}^2),
\end{equation} which yields \begin{equation}\label{3.23}
    \|\bm{W}_i\|_{op} \leq \frac{1}{n}\|\bm{X}_i\|_2 \|\bm{\Lambda\xi}_i\|_2 \leq \frac{3\sqrt{\tr(\bm{\Lambda}^2)}}{2n}\max_{1\leq i\leq n}\|\bm{X}_i\|_2:=L.
\end{equation} 
Moreover, we estimate the matrix variance statistic. First, using (\ref{Lamxi}) in $(i)$ below we have  \begin{equation}\nonumber
    \begin{aligned}
        \sum_{i=1}^n \mathbbm{E}(\bm{W}_i\bm{W}_i^\top) =  \sum_{i=1}^n \frac{\mathbbm{E}\big(\|\bm{\Lambda\xi}_i\|_2^2\big)}{n^2}\bm{X}_i\bm{X}_i^\top \stackrel{(i)}{\preceq} \frac{9\tr(\bm{\Lambda}^2)}{4n^2}\sum_{i=1}^n\bm{X}_i\bm{X}_i^\top ,
    \end{aligned}
\end{equation}
which gives \begin{equation}
    \label{MB1}
    \Big\|\sum_{i=1}^n \mathbbm{E}(\bm{W}_i\bm{W}_i^\top)\Big\|_{op}\leq \frac{9\tr(\bm{\Lambda}^2)}{4n^2}\Big\|\sum_{i=1}^n\bm{X}_i\bm{X}_i^\top\Big\|_{op}.
\end{equation} Second, using $\mathbbm{E}(\bm{\xi}_i\bm{\xi}_i^\top)=\frac{1}{4}\bm{I}_d$ in $(i)$ below we obtain \begin{equation}\nonumber
    \sum_{i=1}^n \mathbbm{E}(\bm{W}_i^\top\bm{W}_i)=\sum_{i=1}^n \frac{\|\bm{X}_i\|_2^2}{n^2}\mathbbm{E}\big(\bm{\Lambda\xi}_i\bm{\xi}_i^\top\bm{\Lambda}\big)\stackrel{(i)}{=}\frac{\sum_{i=1}^n\|\bm{X}_i\|_2^2}{4n^2}\bm{\Lambda}^2,
\end{equation}
which implies \begin{equation}\label{MB2}
    \Big\|\sum_{i=1}^n\mathbbm{E}(\bm{W}_i^\top\bm{W}_i)\Big\|_{op}=\frac{\sum_{i=1}^n\|\bm{X}_i\|_2^2}{4n^2}\|\bm{\Lambda}\|^2_{op}.
\end{equation} Thus,  recall (\ref{3.23}), (\ref{MB1}) and (\ref{MB2}), we can define\begin{equation}\label{3.26}
    \sigma^2:= \frac{9}{4n^2}\left(\tr(\bm{\Lambda}^2)\cdot\Big\|\sum_{i=1}^n\bm{X}_i\bm{X}_i^\top\Big\|_{op}+ \|\bm{\Lambda}\|_{op}^2\cdot \sum_{i=1}^n\|\bm{X}_i\|_2^2\right)
\end{equation}
and then invoke matrix Bernstein's inequality (Lemma \ref{lem2}) to obtain \begin{equation}\nonumber
    \mathbbm{P}\left(\Big\|\sum_{i=1}^n\bm{W}_i\Big\|_{op}\geq u_4\right)\leq  2d \exp\left(-C_8 \min \Big\{\frac{u_4^2}{\sigma^2},\frac{u_4}{L}\Big\}\right)
\end{equation}
for any $u_4\geq 0$. Setting $u_4=C_9[\sigma\sqrt{\log (nd)}+L\log (nd)]$ with large enough $C_9$  yields  
\begin{equation}
    \label{sigmaL}\mathbbm{P}\left(\Big\|\sum_{i=1}^n\bm{W}_i\Big\|_{op}\leq C_9\big[\sigma\sqrt{\log (nd)}+L\log (nd)\big]\right)\geq 1-2(nd)^{-10},
\end{equation}
where $L$ and $\sigma$ are defined in (\ref{3.23}) and (\ref{3.26}), respectively.

\noindent{\it 3.2) Dealing with the randomness of $\{\bm{X}_i\}_{i=1}^n$}

\noindent
Note that $\sigma$ and $L$   depend on $\{\bm{X}_i\}_{i=1}^n$, so we still need to bound them using the randomness of $\{\bm{X}_i\}$. Recall that, on the high-probability event in (\ref{3.18}), $\|\bm{\Lambda}\|_{op}$ and $\tr(\bm{\Lambda}^2)$ can be bounded as in (\ref{3.19})-(\ref{3.20}). Combining with $$ \max_{1\leq i\leq n}\|\bm{X}_i\|_2= \max_{1\leq i\leq n}\sqrt{\sum_{j=1}^n X_{ij}^2}\leq \sqrt{\sum_{j=1}^n \max_{1\leq i\leq n}X_{ij}^2}=\sqrt{\tr(\bm{\Lambda}^2)},$$
we substitute (\ref{3.20}) to obtain \begin{equation}\label{Lbound}
    L\leq \frac{3\tr(\bm{\Lambda}^2)}{2n}\lesssim \frac{K^2\tr(\bm{\Sigma})\log(nd)}{n}.
\end{equation} For bounding $\sigma^2$ in  (\ref{3.26}) we still need to control $\|\sum_{i=1}^n\bm{X}_i\bm{X}_i^\top\|_{op}$ and $\sum_{i=1}^n\|\bm{X}_i\|_2^2$.  Recalling from (\ref{I1bound}) that we have $\big\|\frac{1}{n}\sum_{i=1}^n\bm{X}_i\bm{X}_i^\top - \bm{\Sigma}\big\|_{op}\lesssim  \|\bm{\Sigma}\|_{op}\big(\sqrt{\frac{\hat{K}^2r(\bm{\Sigma})}{n}}+\frac{\hat{K}^2r(\bm{\Sigma})}{n}\big)$ with the promised probability, and note that this implies \begin{equation}
\label{3.31}    \Big\|\sum_{i=1}^n\bm{X}_i\bm{X}_i^\top\Big\|_{op}\leq \Big\|\sum_{i=1}^n\bm{X}_i\bm{X}_i^\top-n\bm{\Sigma}\Big\|_{op}+n\|\bm{\Sigma}\|_{op}\lesssim \big(n+\hat{K}^2r(\bm{\Sigma})\big) \|\bm{\Sigma}\|_{op}.
\end{equation}
To bound $\sum_{i=1}^n\|\bm{X}_i\|_2^2$,    we   write it as \begin{equation}
    \begin{aligned}\label{3.30}
        \sum_{i=1}^n \|\bm{X}_i\|_2^2 = \sum_{i=1}^n \sum_{j=1}^d X_{ij}^2\stackrel{(i)}{=}\sum_{j=1}^d \Sigma_{jj}\left(\sum_{i=1}^n \Big(\frac{X_{ij}}{\|X_{ij}\|_{L^2}}\Big)^2\right)\stackrel{(ii)}{=}\sum_{j=1}^d\Sigma_{jj}\|\bm{Y}_j\|_2^2,
    \end{aligned}
\end{equation}
where we interchange the summation order and use $\Sigma_{jj}=\mathbbm{E}(X_{ij}^2)=\|X_{ij}\|^2_{L^2}$ in $(i)$, and in $(ii)$ we let $$\bm{Y}_j=\|X_{ij}\|_{L^2}^{-1}\big(X_{1j},...,X_{nj}\big)^\top.$$ Since the coordinates are from independent samples, $\{\bm{Y}_j\}_{j=1}^d$ are random vectors with independent sub-Gaussian coordinates  and they satisfy $\big\|\frac{X_{ij}}{\|X_{ij}\|_{L^2}}\big\|_{\psi_2}\leq K$. Hence, we can invoke Lemma \ref{lem4} and a union bound over $1\leq j\leq d$ to obtain that, for any $u_5\geq 0$,
\begin{equation}\nonumber
    \mathbbm{P}\left(\max_{1\leq j\leq d}\Big|\|\bm{Y}_j\|_2-\sqrt{n}\Big|\geq u_5\right)\leq 2d\exp \left(-\frac{C_{10}u_5^2}{\hat{K}^2}\right).
\end{equation}
We let $u_5=C_{11}\hat{K}\sqrt{\log (nd)}$ with sufficiently large $C_{11}$, then   we obtain
$$\mathbbm{P}\Big(\max_{1\leq j\leq d}\big|\|\bm{Y}_j\|_2-\sqrt{n}\big|\lesssim \hat{K}\sqrt{\log(n d)}\Big)\ge 1-2(nd)^{-10},$$
  which implies $\max_{1\leq j\leq d}\|\bm{Y}_j\|_2\lesssim\sqrt{n}$ 
under the scaling of $n\gtrsim \hat{K}^2\log (nd)$. Substituting this into (\ref{3.30}), we arrive at \begin{equation}\label{3.34}
    \sum_{i=1}^n \|\bm{X}_i\|_2^2 \leq \sum_{j=1}^d\Sigma_{jj}C_{12}n=C_{12}n\cdot \tr(\bm{\Sigma}). 
\end{equation}
Then, we combine (\ref{3.19}), (\ref{3.20}), (\ref{3.31}), (\ref{3.34}) to bound $\sigma^2$ in  (\ref{3.26}) as \begin{equation}\label{sigmabound}
    \sigma^2 \lesssim \frac{K^2\|\bm{\Sigma}\|_{op}\tr(\bm{\Sigma})\log (nd)}{n}+\frac{K^2\hat{K}^2 [\tr(\bm{\Sigma})]^2\log(nd)}{n^2}.
\end{equation}
Further, substituting (\ref{Lbound}), (\ref{sigmabound}) into (\ref{sigmaL}),  along with $\|I_3\|_{op}\le 2\|\sum_{i=1}^n\bm{W}_i\|_{op}$, we obtain   \begin{equation}\label{I3bound}
    \|I_3\|_{op}\lesssim     \sqrt{\frac{K^2\|\bm{\Sigma}\|_{op}\tr(\bm{\Sigma})\log^2(nd)}{n}} +  \frac{K^2 \tr(\bm{\Sigma})\log^2(nd)}{n} 
\end{equation} holds 
with probability at least $1-2(nd)^{-10}$.

\subsubsection*{Step 4: Combining Everything}

Substituting the bounds in  (\ref{I1bound}), (\ref{I2bound}), (\ref{I3bound}) into (\ref{finalback}), we obtain that  \begin{equation}\nonumber
    \|\bm{\widetilde{\Sigma}}-\bm{\Sigma}\|_{op} \lesssim K^2 \left(\sqrt{\frac{\tr(\bm{\Sigma})\|\bm{\Sigma}\|_{op} \log^2(nd)}{n}}+ \frac{\tr(\bm{\Sigma})\log^2(nd)}{n}\right)
\end{equation}
holds with the promised probability. Recalling $\tr(\bm{\Sigma})=r(\bm{\Sigma})\|\bm{\Sigma}\|_{op},$ the result follows.  \end{proof}

\section{Concluding Remarks}\label{sec6}

For sub-Gaussian distributions, we propose a 2-bit covariance estimator $\bm{\widetilde{\Sigma}}$  (\ref{pfesti}) that possesses near-optimal dimension-free operator norm error rate up to logarithmic factors; c.f.,  Theorem \ref{maintheorem} and Theorem \ref{lem3}. This  improves on $\widehat{\bm{\Sigma}}_{na}$ and  $\widehat{\bm{\Sigma}}_{a}$ proposed  in \cite{dirksen2022covariance,dirksen2023tuning} that are sub-optimal for $\bm{\Sigma}$ in the case of $\tr(\bm{\Sigma})\ll d\|\bm{\Sigma}\|_\infty$. More practically, our second  advantage  is that our estimator is free of {\it any} tuning parameter, unlike $\bm{\widehat{\Sigma}}_{na}$ (\ref{aosesti}) in \cite{dirksen2022covariance} that requires tuning according to certain unknown parameters.

As significant departures from \cite{dirksen2022covariance,dirksen2023tuning} who used uniformly dithered 1-bit quantizer, our quantization procedure is motivated by \cite{chen2022quantizing} and built upon  a 2-bit quantizer associated with triangular dither. This immediately leads to a new 2-bit estimator $\bm{\widetilde{\Sigma}}_{\tb}$ (\ref{2.16}) comparable to $\bm{\widehat{\Sigma}}_{na}$ in several respects, and further, our key estimator $\bm{\widetilde{\Sigma}}$ is obtained from  $\bm{\widetilde{\Sigma}}_{\tb}$  by using dithering scales $(\lambda_{j,\min}=\max_{1\le i\le n}|X_{ij}|)_{j=1}^d$ which vary across entries and are determined by  the data.

Our simulations under Gaussian samples demonstrate  the following: (i) $\bm{\widetilde{\Sigma}}_{\tb}$ outperforms $\bm{\widehat{\Sigma}}_{na}$ noticeable; (ii) our parameter-free $\bm{\widetilde{\Sigma}}$   outperforms $\bm{\widehat{\Sigma}}_{na},\bm{\widetilde{\Sigma}}_{\tb}$ with {\it optimally} tuned parameter in the case of $\tr(\bm{\Sigma})\ll d\|\bm{\Sigma}\|_\infty$; (iii) the performance of  $\bm{\widetilde{\Sigma}}$ can be further improved by properly shrinking the dithering scales in (\ref{3.1}). For (iii), it is remarkable that $\bm{\widetilde{\Sigma}}(0.5)$ using $(\frac{1}{2}\lambda_{j,\min})_{j=1}^d$ as dithering scales typically achieves operator norm errors less than twice of the errors of sample covariance. In the theoretical analysis, we deal with  deal with the dependence between the dithering scales and samples by
first conditioning on the samples and 
utilizing the randomness of the dither, and then using matrix Bernstein's inequality to establish dimension-free bound.

Let us close this paper with some discussions on several   aspects.   
\paragraph{{\color{black}Online Setting:}} We continue from Remark \ref{online} and adapt our estimator to the online setting where the samples arise sequentially and one cannot obtain $\bm{\Lambda}_{\min}$. We first observe that the proof of Theorem \ref{maintheorem} does not essentially relies on the specific choice of dithering scales $(\lambda_{j,\min})_{j=1}^d$. In fact, all the arguments and thus the final guarantee are valid for any dithering scales $(\lambda_j)_{j=1}^d$ possibly depending on the samples $(\bm{X}_i)_{i=1}^d$  and   satisfying\footnote{Specifically, $\lambda_j\ge\lambda_{j,\min}$ is needed in equating the 2-bit quantizer $\mathcal{Q}_{\lambda_j,\tb}$ and $\mathcal{Q}_{\lambda_j}$ as per (\ref{3.2}), and $\lambda_j\lesssim K\sqrt{\Sigma_{jj}\log(nd)}$ is required in controlling $\|\bm{\Lambda}\|_{op}$ and $\tr(\bm{\Lambda}^2)$ as per (\ref{3.19})--(\ref{3.20}).} 
\begin{align}
    \lambda_{j,\min} \le \lambda_j \lesssim K \sqrt{\Sigma_{jj}\log(nd)},\quad j=1,2,...,d.\label{desir}
\end{align}
 In light of this, our remedy for the online setting is to first obtain such $\tilde{\lambda}_j$'s from  a {\it small} number of samples, and then use them for subsequent quantization and estimation. Specifically, we propose to first collect 
 $\bm{X}_1,\bm{X}_2,...,\bm{X}_{n_0}$ with $n_0=O_K(\log(nd))$ and define 
 \begin{align}\label{ondither}
     \tilde{\lambda}_j= C(K)\left[\Big(\frac{1}{n_0}\sum_{i=1}^{n_0}X_{ij}^2 \Big)\log(nd)\right]^{1/2},\quad j=1,2,...,d
 \end{align}
 for some large enough tuning parameter $C(K)$ depending on $K$ only. Note that these $\tilde{\lambda}_j$ satisfy (\ref{desir}) due to Bernstein's inequality; see Appendix C for details. We then use 
$\widetilde{\bm{\Lambda}}=\diag(\tilde{\lambda}_1,...,\tilde{\lambda}_d)$ and $\bm{\tau}_i\sim \mathscr{U}[-\frac{1}{2},\frac{1}{2}]^d+\mathscr{U}[-\frac{1}{2},\frac{1}{2}]^d$
 for quantizing subsequent samples $\bm{X}_{n_0+1},...,\bm{X}_n$ to 
 \begin{align*}
     \dot{\bm{X}}_i  = \bm{\mathcal{Q}}_{\widetilde{\bm{\Lambda}},\tb}(\bm{X}_i+\widetilde{\bm{\Lambda}}\bm{\tau}_i) = \Big(\mathcal{Q}_{\tilde{\lambda}_1,\tb}(X_{i1}+\tilde{\lambda}_1\tau_{i1}),...,\mathcal{Q}_{\tilde{\lambda}_d,\tb}(X_{id}+\tilde{\lambda}_d\tau_{id})\Big)^\top,\quad i =n_0+1,...,n.
 \end{align*}
 From $(\bm{X}_i)_{i=1}^{n_0}$ and $(\dot{\bm{X}}_i)_{i=n_0+1}^n$, we propose the estimator 
 \begin{align}\label{onlineest}
     \widetilde{\bm{\Sigma}}_{on} = \frac{1}{n}\Big[\sum_{i=1}^{n_0}\bm{X}_i\bm{X}_i^\top+\sum_{i=n_0+1}^n \dot{\bm{X}}_i\dot{\bm{X}}_i^\top\Big] - \frac{n-n_0}{4n}\widetilde{\bm{\Lambda}}^2
 \end{align}
 and can show     $\widetilde{\bm{\Sigma}}_{on}$ satisfies the same dimension-free error rate as $\widetilde{\bm{\Sigma}}$; see Appendix C for details.

Note that there can certainly be other adaptations built on this idea. For instance, in a setting where it is prohibitive to obtain unquantized samples, one can first use quantization methods with a common large enough dithering scale  (i.e., the methods for  $\widehat{\bm{\Sigma}}_{na},~\widehat{\bm{\Sigma}}_{\tb}$) to find suitable $\tilde{\lambda}_j$'s, and then switch to our quantization method with dithering scales $\tilde{\lambda}_j$'s.

By testing $d=10$ and $n=400:200:1600$, we report simulation results to show that the ``online estimator'' $\widetilde{\bm{\Sigma}}_{on}$ behaves comparably to $\{\widetilde{\bm{\Sigma}}(s):s=0.5,0.7,0.9\}$ over $\bm{\Sigma}=\bm{\Sigma}(1,0.2,25)$; see Figure \ref{fig3} and the associated texts for details.

\begin{figure}[ht!]
    \centering
    \includegraphics[scale = 0.7]{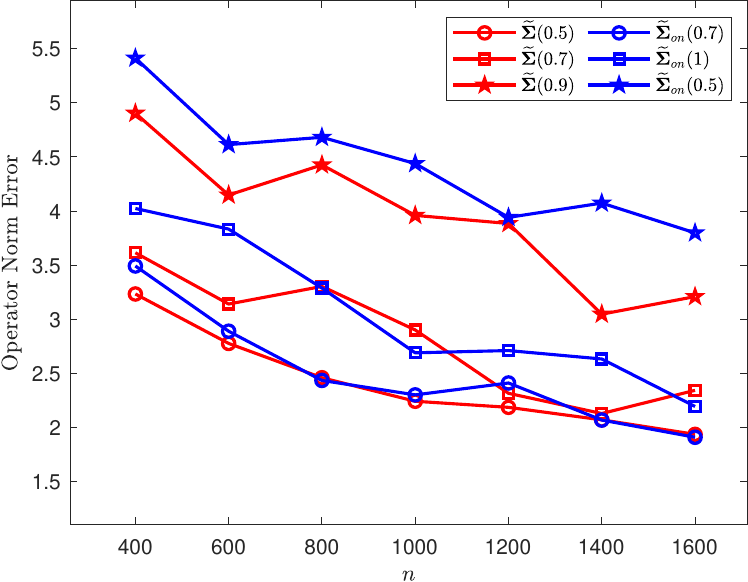}

    \caption{\footnotesize Comparing
    the ``online'' estimator $\widetilde{\bm{\Sigma}}_{on}$ and       estimators $\{\widetilde{\bm{\Sigma}}(s):s=0.5,0.7,0.9\}$ under Gaussian samples with covariance matrix $\bm{\Sigma}=\bm{\Sigma}(1,0.2,25)$. 
    Let us use $n_0 =\lceil C_1(K)\log (n)\rceil$ and $\tilde{\lambda}_j =C(K)\big[(\frac{1}{n_0}\sum_{i=1}^{n_0}X_{ij}^2)\log (n)\big]^{1/2}$ for $\widetilde{\bm{\Sigma}}_{on}$. We  simply fix $C_1(K)=3$ and notice that we only collect very few samples for determining $\tilde{\lambda}_j$, e.g., we use $n_0 = \lceil 3\log(1600)\rceil=23$ when $n=1600$. Regarding $C(K)$, we find that $C(K)=0.7$ is empirically near-optimal under $n=500$. In the simulation, we test such near-optimal choice $C(K)=0.7$ and also the sub-optimal choices $C(K)=0.5$ and $C(K)=1$, with the corresponding estimators denoted by $\widetilde{\bm{\Sigma}}_{on}(0.7)$, $\widetilde{\bm{\Sigma}}_{on}(0.5)$ and $\widetilde{\bm{\Sigma}}_{on}(1)$, respectively. We find that $\widetilde{\bm{\Sigma}}_{on}$ with near-optimal $C(K)$ performs comparably to $\widetilde{\bm{\Sigma}}(0.5)$.} 
    \label{fig3}
\end{figure}

\paragraph{{\color{black} Probability Terms:}} In this paper, we present the   probability terms for the theoretical guarantees of the 2-bit estimators consistently as $1-O((nd)^{-10})$ (note that $3\exp(-\frac{d}{2})$ in Theorem \ref{thm4} and $10\exp(-10r(\bm{\Sigma}))$ in Theorem \ref{maintheorem} are typically much smaller than $(nd)^{-10}$). Such probability is weaker than the one for existing covariance estimation guarantees without quantization; e.g., see Theorem \ref{lem3} which immediately implies the bound $O\big(\sqrt{\frac{\tr(\bm{\Sigma})\|\bm{\Sigma}\|_{op}}{n}}+\frac{\tr(\bm{\Sigma})}{n}\big)$ with probability exceeding $1-3\exp(-r(\bm{\Sigma}))$; see also \cite{koltchinskii2017concentration,abdalla2024covariance}, among others. While we pay not attention to refining the probability in this work, we also point out that   finer probability terms are possible under more careful analysis. As an example, in Theorem \ref{thm4} we can bound $\|\widetilde{\bm{\Sigma}}_{\tb}-\bm{\Sigma}\|_{op}$ by bounding $\|\widetilde{\bm{\Sigma}}_{\tb}-\mathbbm{E}\widetilde{\bm{\Sigma}}_{\tb}\|_{op}$ and $\|\mathbbm{E}\widetilde{\bm{\Sigma}}_{\tb}-\bm{\Sigma}\|_{op}$ separately, which avoids the probability term of $O((nd)^{-10})$ and will instead lead to a probability tail similar to Theorem \ref{lem3}. It is interesting to pursue similar improvement for Theorem \ref{maintheorem}. 

\paragraph{Two Quantizers:} While the numerical improvement of our $\bm{\widetilde{\Sigma}}_{\tb}$ over $\bm{\widehat{\Sigma}}_{na}$  is notable over Gaussian samples (Figure \ref{fig2}), this is in general not     theoretically provable   since it may not be true for some other distribution.  As an extreme example,\footnote{It is extreme in the sense that the samples are already binary and thus there is no quantization need.} suppose that entries of $\bm{X}_i$ are $\{\pm 1\}$-valued, then $\bm{\widehat{\Sigma}}_{na}$ in (\ref{aosesti}) with     $\lambda=1$ simply reduces to sample covariance, meaning that there is no information loss in the quantization. Thus, our $\bm{\widetilde{\Sigma}}_{\tb}$ is expected to behave (at least slightly) worse than $\bm{\widehat{\Sigma}}_{na}$. Therefore, $\bm{\widehat{\Sigma}}_{na}$ and $\bm{\widetilde{\Sigma}}_{\tb}$ correspond to two {\it parallel} dithered quantization methods that allow for covariance estimation, with the former applying a uniformly dithered sign quantizer twice, while the latter being the 2-bit quantizer in (\ref{2.15}) associated with a triangular dither;  in general, it may be a rushed judgment to declare one superior to the other.

Built upon $\bm{\widetilde{\Sigma}}_{\tb}$,   our major contribution is to devise $\bm{\widetilde{\Sigma}}$   that enjoys near optimal operator norm rate and requires no tuning, which is done by
using dithering scales   in (\ref{3.1}). It is interesting to note that this development straightforwardly carries over to $\bm{\widehat{\Sigma}}_{na}$. In particular, let $\bm{\tau}_{i1},\bm{\tau}_{i2}\sim \mathscr{U}[-1,1]^d$ and $\bm{\Lambda}_{\min}=\diag(\lambda_{1,\min},...,\lambda_{d,\min})$ where $\lambda_{j,\min}$ is given in (\ref{3.1}), we consider the following {\it parameter-free} estimator \begin{equation}\label{aosestiada}
\begin{aligned}
    \bm{\widehat{\Sigma}}^\star= \frac{1}{2n}\sum_{i=1}^n (\bm{\dot{X}}_{i1}\bm{\dot{X}}_{i2}^\top+\bm{\dot{X}}_{i2}\bm{\dot{X}}_{i1}^\top),\quad\text{with }\begin{cases}\nonumber
        \bm{\dot{X}}_{i1}=\bm{\Lambda}_{\min}\cdot \sign(\bm{X}_i+\bm{\Lambda}_{\min}\bm{\tau}_{i1})\\
        \bm{\dot{X}}_{i2}=\bm{\Lambda}_{\min}\cdot \sign(\bm{X}_i+\bm{\Lambda}_{\min}\bm{\tau}_{i2})
    \end{cases}.
    \end{aligned}
\end{equation}
Then, using techniques provided in the proof of Theorem \ref{maintheorem}, one can show that $\bm{\widehat{\Sigma}}^\star$ possesses a {\it near-optimal} operator norm error rate that is comparable to Theorem \ref{maintheorem} for $\bm{\widetilde{\Sigma}}$ and only logarithmically worse than Theorem \ref{lem3} for sample covariance. Additionally, we expect that   properly shrinking the dithering scales like (\ref{shrinkageesti}) also improves the numerical performance of $\bm{\widehat{\Sigma}}^\star$ over Gaussian samples.



\subsection*{Appendix A: The Proof of Theorem \ref{lem3}}\label{sec5}
The proof is adjusted from \cite[Thm. 9.2.4]{vershynin2018high}. The key technical tool is the matrix deviation inequality with optimal dependence on $K$ that can be found in \cite{jeong2022sub}.
\begin{lem}{\rm (Corollary 1.2 in \cite{jeong2022sub})}\label{matrixd}
    Consider $\bm{A}=[\bm{a}_1,...,\bm{a}_n]^\top\in \mathbb{R}^{n\times d}$ with independent isotropic rows $\bm{a}_i$'s satisfying $\max_{1\leq i\leq n}\|\bm{a}_i\|_{\psi_2}\leq K$ and a bounded subset $\mathcal{T}\subset \mathbb{R}^d$ with radius $\rad(\mathcal{T}):=\sup_{\bm{x}\in \mathcal{T}}\|\bm{x}\|_2$, Gaussian width $\omega(\mathcal{T}):=\mathbbm{E}\sup_{\bm{x}\in\mathcal{T}}\langle\bm{g},\bm{x}\rangle$ for $\bm{g}\sim \mathcal{N}(0,\bm{I}_d)$. Let $\hat{K}=K\sqrt{\log K}$, then for any $t\geq 0$, with probability at least $1-3e^{-t^2}$ we have 
    \begin{equation}\nonumber
        \sup_{\bm{x}\in \mathcal{T}}\big|\|\bm{Ax}\|_2-\sqrt{n}\|\bm{x}\|_2\big| \leq C\hat{K}\big(\omega(\mathcal{T})+t\cdot\rad(\mathcal{T})\big).
    \end{equation}
\end{lem}

\begin{proof}[The proof of Theorem \ref{lem3}]
We first consider the case of $\rank(\bm{\Sigma})=d$ and let $\bm{Y}_i = \bm{\Sigma}^{-\frac{1}{2}}\bm{X}_i$. Note that $\bm{Y}_i$ is isotropic since $\mathbbm{E}(\bm{Y}_i\bm{Y}_i^\top)=\mathbbm{E}(\bm{\Sigma}^{-\frac{1}{2}}\bm{X}_i\bm{X}_i^\top\bm{\Sigma}^{-\frac{1}{2}})=\bm{I}_d$, and further, its  $\psi_2$ norm is bounded by $K$ because
\begin{equation}
\begin{aligned}\nonumber
    \|\bm{Y}_i\|_{\psi_2}&=\sup_{\bm{a}\in \mathbb{S}^{d-1}}\|\langle \bm{a},\bm{\Sigma}^{-\frac{1}{2}}\bm{X}_i\rangle\|_{\psi_2}=\sup_{\bm{a}\in \mathbb{S}^{d-1}}\|\langle \bm{\Sigma}^{-\frac{1}{2}}\bm{a},\bm{X}_i\rangle\|_{\psi_2}\\
    &\leq \sup_{\bm{a}\in \mathbb{S}^{d-1}}K\|\langle \bm{\Sigma}^{-\frac{1}{2}}\bm{a},\bm{X}_i\rangle\|_{L^2}=\sup_{\bm{a}\in \mathbb{S}^{d-1}}K\|\langle \bm{a},\bm{Y}_i\rangle\|_{L^2}=K.
    \end{aligned}
\end{equation}
Then, we let $\bm{Y}=[\bm{Y}_1,...,\bm{Y}_n]^\top\in \mathbb{R}^{n\times d}$ and can write the operator norm error as \begin{equation}
    \begin{aligned}\label{4.3}
        &\left\|\frac{1}{n}\sum_{i=1}^n\bm{X}_i\bm{X}_i^\top-\bm{\Sigma}\right\|_{op}  = \sup_{\bm{a}\in \mathbb{S}^{d-1}}\left| \bm{a}^\top\left(\frac{1}{n}\sum_{i=1}^n\bm{X}_i\bm{X}_i^\top-\bm{\Sigma}\right)\bm{a}\right|\\
        &=\sup_{\bm{a}\in \mathbb{S}^{d-1}}\frac{1}{n}\left|\sum_{i=1}^n\big|\big<\bm{Y}_i,\bm{\Sigma}^{\frac{1}{2}}\bm{a}\big>\big|^2-n\big\|\bm{\Sigma}^\frac{1}{2}\bm{a}\big\|_2^2\right| \stackrel{(i)}{=}\sup_{\bm{b}\in\bm{\Sigma}^{\frac{1}{2}}\mathbb{S}^{d-1}}\frac{1}{n}\Big|\|\bm{Yb}\|_2^2-n\|\bm{b}\|^2_2\Big|\\
        &\leq \sup_{\bm{b}\in \bm{\Sigma}^{\frac{1}{2}}\mathbb{S}^{d-1}}\frac{1}{n}\Big|\|\bm{Yb}\|_2-\sqrt{n}\|\bm{b}\|_2\Big|\Big|\|\bm{Yb}\|_2+\sqrt{n}\|\bm{b}\|_2\Big|,
    \end{aligned}
\end{equation}
where we use a change of variable $\bm{b}=\bm{\Sigma}^{\frac{1}{2}}\bm{a}$ in $(i)$.
Now we invoke Lemma \ref{matrixd} to obtain that, for any $t\geq 0$, with probability at least $1-3e^{-t^2}$ 
\begin{equation}
    \sup_{\bm{b}\in \bm{\Sigma}^{\frac{1}{2}}\mathbb{S}^{d-1}}\Big|\|\bm{Yb}\|_2-\sqrt{n}\|\bm{b}\|_2\Big| \leq C\hat{K}\Big[\omega(\bm{\Sigma}^{\frac{1}{2}}\mathbb{S}^{d-1})+t\cdot\rad(\bm{\Sigma}^{\frac{1}{2}}\mathbb{S}^{d-1})\Big].\label{4.4}
\end{equation}
Note that $\rad(\bm{\Sigma}^{\frac{1}{2}}\mathbb{S}^{d-1})\leq \|\bm{\Sigma}\|_{op}^{\frac{1}{2}}$, and let $\bm{g}\sim \mathcal{N}(0,\bm{I}_d)$ we have $$\omega\big(\bm{\Sigma}^{\frac{1}{2}}\mathbb{S}^{d-1}\big)=\mathbbm{E}\sup_{\bm{a}\in \mathbb{S}^{d-1}}\langle\bm{g},\bm{\Sigma}^{\frac{1}{2}}\bm{a}\rangle=\mathbbm{E}\big\|\bm{\Sigma}^{\frac{1}{2}}\bm{g}\big\|_2 \leq \sqrt{\mathbbm{E}\|\bm{\Sigma}^{\frac{1}{2}}\bm{g}\|_2^2}=\sqrt{\tr(\bm{\Sigma})}.$$
Thus,   (\ref{4.4}) yields 
\begin{equation}
    \label{111}
    \sup_{\bm{b}\in \bm{\Sigma}^{\frac{1}{2}}\mathbb{S}^{d-1}}\Big|\|\bm{Yb}\|_2-\sqrt{n}\|\bm{b}\|_2\Big|\leq  C\hat{K}\Big[\sqrt{\tr(\bm{\Sigma})}+t\sqrt{\|\bm{\Sigma}\|_{op}}\Big],
\end{equation}
which further implies \begin{equation}
    \begin{aligned}\label{222}
    \sup_{\bm{b}\in \bm{\Sigma}^{\frac{1}{2}}\mathbb{S}^{d-1}}\Big|\|\bm{Yb}\|_2+\sqrt{n}\|\bm{b}\|_2\Big|&\leq \sup_{\bm{b}\in \bm{\Sigma}^{\frac{1}{2}}\mathbb{S}^{d-1}}\Big|\|\bm{Yb}\|_2-\sqrt{n}\|\bm{b}\|_2\Big| +2\sqrt{n}\rad\big(\bm{\Sigma}^{\frac{1}{2}}\mathbb{S}^{d-1}\big)\\
    &\leq  C\hat{K}\Big[\sqrt{\tr(\bm{\Sigma})}+t\sqrt{\|\bm{\Sigma}\|_{op}}\Big]+2 \sqrt{n\|\bm{\Sigma}\|_{op}}.
    \end{aligned}
\end{equation}
Substituting (\ref{111}) and (\ref{222}) into (\ref{4.3}), then taking $t=\sqrt{u}$, some algebra immediately yields the desired bound.

All that remains is to deal with the case when $\rank(\bm{\Sigma}):=r<d$. Let the singular value decomposition be $\bm{\Sigma}=\bm{U}\bm{\Sigma}_0\bm{U}^\top$ for some $\bm{U}\in \mathbb{R}^{d\times r}$ with orthonormal columns and positive definite diagonal $\bm{\Sigma}_0\in \mathbb{R}^{r\times r}$. Then we construct the isotropic and $K$-sub-Gaussian $\bm{Y}_i=\bm{\Sigma}_0^{-\frac{1}{2}}\bm{U}^\top\bm{X}_i$. Note that $$\mathbbm{E}\|(\bm{I}_n-\bm{UU}^\top)\bm{X}_i\|_2^2 = (\bm{I}_n-\bm{UU}^\top)\bm{U\Sigma}_0\bm{U}^\top(\bm{I}_n-\bm{UU}^\top)=0,$$ hence almost surely we have $\bm{X}_i=\bm{UU}^\top\bm{X}_i=\bm{U\Sigma}^{\frac{1}{2}}_0\bm{Y}_i$, where we use $\bm{U}^\top\bm{X}_i = \bm{\Sigma}^{\frac{1}{2}}\bm{Y}_i$ in the second equality. Thus, we have\begin{equation}
    \bm{X}_i=\bm{U}_0\bm{Y}_i,\text{  where }\bm{U}_0=\bm{U}\bm{\Sigma}_0^{\frac{1}{2}}\in \mathbb{R}^{d\times r}.\label{5.5}
\end{equation}     Moreover, we note     $\rad(\bm{U}_0^\top\mathbb{S}^{d-1})\leq \|\bm{U}_0\|_{op} \leq  \|\bm{\Sigma}\|_{op}^{1/2}$ and bound the Gaussian width as (let $\bm{g}\sim \mathcal{N}(0,\bm{I}_d)$) \begin{equation}
    \begin{aligned}\nonumber
        \omega(\bm{U}_0^\top \mathbb{S}^{d-1})&= \mathbbm{E}\sup_{\bm{a}\in \mathbb{S}^{d-1}}\bm{g}^\top\bm{U}_0^\top\bm{a} = \mathbbm{E}\|\bm{U}_0\bm{g}\|_2 \leq \sqrt{\mathbbm{E}\|\bm{U}_0\bm{g}\|_2^2}=\sqrt{\tr(\bm{\Sigma}_0)}= \sqrt{\tr(\bm{\Sigma})}.
    \end{aligned}
\end{equation} With these preparations, one can plug (\ref{5.5}) into the first line of (\ref{4.3}) and run the mechanism again to obtain the  claim.\end{proof}
\subsection*{Appendix B: The Proof of Lemma \ref{lem1}}
To be self-contained, we give a proof for Lemma \ref{lem1}. Our proof is based on the following result from \cite{gray1993dithered}, and similar materials can be found in \cite[Sec. II.B]{chen2022quantizing}. 
\begin{lem}\label{thm6}
    {\rm(Adapted from \cite[Thms. 1-2]{gray1993dithered})} We let $a\in\mathbb{R}$,   $\tau$ be some random dither and quantize $a$ to $\mathcal{Q}_\lambda(a+\tau)$, then we define quantization error as $w=\mathcal{Q}_\lambda(a+\tau)-(a+\tau)$, and quantization noise as $\xi=\mathcal{Q}_\lambda(a+\tau)-a$. We assume that $\tau'\sim \mathscr{U}[-\frac{\lambda}{2},\frac{\lambda}{2}]$ is independent of $\tau$. Denote by $\imath$ the complex unit, then we have the following:
    \begin{itemize}
        \item If $f(u)=\mathbbm{E}(\exp(\imath u\tau))$ satisfies $f(\frac{2\pi l}{\lambda})=0$ for any non-zero integer $l$, then for any $a\in\mathbb{R}$, we have $w\sim \mathscr{U}[-\frac{\lambda}{2},\frac{\lambda}{2}]$;
        \item If $g(u)=\mathbbm{E}(\exp(\imath u\tau))\mathbbm{E}(\exp(\imath u\tau'))$ satisfies $g''(\frac{2\pi l}{\lambda})=0$ for any non-zero integer $l$, then for any $a\in \mathbb{R}$, we have $\mathbbm{E}(\xi^2)=\mathbbm{E}(\tau+\tau')^2$. 
    \end{itemize}
\end{lem}
Now we are ready to prove Lemma \ref{lem1}.

\begin{proof}[The Proof of Lemma \ref{lem1}] Observe that $\sup_{x\in\mathbb{R}}|\mathcal{Q}_\lambda(x)-x|=\frac{\lambda}{2}$, so the boundedness $(i)$ follows from $
    \|\bm{\xi}\|_\infty \leq \|\mathcal{Q}_\lambda(\bm{X}+\bm{\tau})-(\bm{X}+\bm{\tau})\|_\infty + \|\bm{\tau}\|_\infty\leq \frac{\lambda}{2}+\lambda=\frac{3\lambda}{2}.$ 
For (ii) $\mathbbm{E}[\bm{\xi}]=0$ and  (iv) $\mathbbm{E}[\bm{\xi}\bm{\xi}^\top]=\frac{\lambda^2}{4}\bm{I}_d$, it suffices to show that they hold under a fixed $\bm{X}\in \mathbb{R}^d$. We let $\bm{w}=\mathcal{Q}_\lambda(\bm{X}+\bm{\tau})-(\bm{X}+\bm{\tau})$ be the quantization error and decompose the quantization noise as $\bm{\xi}=\bm{\tau}+\bm{w}$.

We now verify that a triangular dither $\tau=\tau_1+\tau_2$ with independent $\tau_1,\tau_2\sim \mathscr{U}[-\frac{\lambda}{2},\frac{\lambda}{2}]$ satisfies both conditions in Lemma \ref{thm6}. First, we calculate that \begin{equation}
    f(u)=\mathbbm{E}(\exp(\imath u\tau))=\mathbbm{E}(\exp(\imath u\tau_1)\exp(\imath u\tau_2))=[\mathbbm{E}(\exp(\imath u\tau_1))]^2,\label{5.6}
\end{equation} and moreover, \begin{equation}
    \label{5.7}
    \mathbbm{E}(\exp(\imath u\tau_1))=\int_{-\lambda/2}^{\lambda/2}\lambda^{-1}[\cos(u\tau_1)+\imath\sin(u\tau_1)]~\mathrm{d}\tau_1 = \frac{2}{\lambda u}\sin\frac{\lambda u}{2},
\end{equation} which vanishes at $u=\frac{2\pi l}{\lambda}$ for non-zero integer $l$. Thus, Lemma \ref{thm6} implies that entries of $\bm{w}$ follow $\mathscr{U}[-\frac{\lambda}{2},\frac{\lambda}{2}]$, which leads to $\mathbbm{E}(\bm{\xi} )=\mathbbm{E}(\bm{\tau} )+\mathbbm{E}(\bm{w})=0$. Second, we let $\tau'\sim \mathscr{U}[-\frac{\lambda}{2},\frac{\lambda}{2}]$ be independent of $\{\tau_1,\tau_2\}$, then using (\ref{5.6}) and (\ref{5.7}) we obtain $$g(u)=f(u)\mathbbm{E}(\exp(\imath u\tau'))=\Big(\frac{2}{\lambda u}\sin(\frac{\lambda u}{2})\Big)^3,$$ which satisfies $g''(\frac{2\pi l}{\lambda})=0$ for any non-zero integer $l$. Thus, let $\xi_i$ be the $i$-th entry of $\bm{\xi}$, Lemma \ref{thm6} implies that $$\mathbbm{E}(\xi_i^2)=\mathbbm{E}(\tau+\tau')^2=\mathbbm{E}(\tau)^2+\mathbbm{E}(\tau')^2=\frac{\lambda^2}{4}.$$ For non-diagonal entries,  $\xi_i$ and $\xi_j$ are independent if $i\neq j$, and so $\mathbbm{E}(\xi_i\xi_j )=\mathbbm{E}(\xi_i )\mathbbm{E}(\xi_j)=0.$ Overall, we have shown that $\mathbbm{E}(\bm{\xi}\bm{\xi}^\top)=\frac{\lambda^2}{4}\bm{I}_d$.

Finally, we prove (iii) $\|\bm{\xi}\|_{\psi_2}=O(\lambda)$ by using $\|X\|_{\psi_2}\asymp \sup_{p\geq 1}\frac{[\mathbbm{E}|X|^p]^{1/p}}{\sqrt{p}}$ (e.g., \cite[Prop. 2.5.2]{vershynin2018high}). Under fixed $\bm{X}$, we have shown   $\|\bm{\xi}\|_\infty\leq \frac{3\lambda}{2}$ and $\mathbbm{E}(\bm{\xi})=0$. Combined with the fact that $\bm{\xi}$ have independent entries, we have $\|\bm{\xi}\|_{\psi_2}=O(\lambda)$ \cite[Lem. 3.4.2]{vershynin2018high}. Thus, we have  $\mathbbm{E}\big[|\langle \bm{v},\bm{\xi}\rangle|^p\big|\leq (C\lambda\sqrt{p})^p$ for some absolute constant $C$ and for any $\bm{v}\in \mathbb{S}^{d-1}$. Further averaging over $\bm{X}$, we obtain $\sup_{\bm{v}\in \mathbb{S}^{d-1}}\mathbbm{E}|\langle \bm{v},\bm{\xi}\rangle|^p\leq (C\lambda\sqrt{p})^p$, which leads to $\|\bm{\xi}\|_{\psi_2}=\sup_{\bm{v}\in \mathbb{S}^{d-1}}\|\langle\bm{v},\bm{\xi}\rangle\|_{\psi_2}=O(\lambda)$. The proof is complete.  \end{proof}
\subsection*{{\color{black}Appendix C: $\widetilde{\bm{\Sigma}}_{on}$ for Online Setting}}
 This appendix collects some missing details in the analysis of our estimator $\widetilde{\bm{\Sigma}}_{on}$ for the online setting. First, we shall show that $\tilde{\lambda}_j$'s  in (\ref{ondither}), which involves a tuning parameter $C(K)$ and $n_0=O_K(\log(nd))$ full samples, satisfy (\ref{desir}). Due to the high-probability event $\{\lambda_{j,\min}\lesssim K\sqrt{\Sigma_{jj}\log(nd)},~\forall j\in[d]\}$ (see (\ref{3.18})),    it suffices to show $\tilde{\lambda}_j \asymp _K \sqrt{\Sigma_{jj}\log(nd)}$ for any $j\in[d]$, and thus it is also enough to show 
 $
     \frac{1}{n_0}\sum_{i=1}^{n_0}X_{ij}^2\asymp \Sigma_{jj}$ for $j\in[d]$. The phenomenon that one can accurately estimate the {\it   (diagonal) entries} of a covariance matrix from {\it logarithmically many} samples is a direct outcome of Bernstein's inequality (e.g., \cite[Thm. 2.8.1]{vershynin2018high}). To see this, by $\|X_{ij}^2\|_{\psi_1}\le \|X_{ij}\|_{\psi_2}^2\le K^2 \Sigma_{jj}$,  Bernstein's inequality gives 
     \begin{align*}
         \mathbbm{P}\Big(\Big|\frac{1}{n_0}\sum_{i=1}^{n_0}X_{ij}^2 - \Sigma_{jj}\Big|\ge t\Big) \le 2\exp\Big(-cn_0 \min\Big\{\frac{t^2}{K^4\Sigma_{jj}^2},\frac{t}{K^2\Sigma_{jj}}\Big\}\Big),\quad \forall t>0,~j=1,...,d.
     \end{align*}
     Setting $t = \frac{1}{2}\Sigma_{jj}$ along with a union bound over $j\in[d]$ yields that, under $n_0 = CK^4 \log(nd)=O_K(\log(nd))$, with high probability $\frac{1}{2}\Sigma_{jj}\le\frac{1}{n_0}\sum_{i=1}^{n_0}X_{ij}^2\le \frac{3}{2}\Sigma_{jj}$ holds for any $j\in[d]$, as desired.

     Next, we outline several key steps in the analysis of  $\widetilde{\bm{\Sigma}}_{on}$ as per (\ref{onlineest}). For $i\ge n_0+1$, under (\ref{desir}) we have $\dot{\bm{X}}_i = \bm{\mathcal{Q}}_{\widetilde{\bm{\Lambda}},\tb}(\bm{X}_i+\widetilde{\bm{\Lambda}}\bm{\tau}_i)= \widetilde{\bm{\Lambda}}\mathcal{Q}_1(\widetilde{\bm{\Lambda}}^{-1}\bm{X}_i+\bm{\tau}_i)$, and by defining the quantization noise $\bm{\xi}_i = \mathcal{Q}_1(\widetilde{\bm{\Lambda}}^{-1}\bm{X}_i+\bm{\tau}_i)-\widetilde{\bm{\Lambda}}^{-1}\bm{X}_i$ we have $\dot{\bm{X}}_i = \bm{X}_i+ \widetilde{\bm{\Lambda}}\bm{\xi}_i$. Therefore, we have 
     \begin{equation}
         \begin{aligned}
             \widetilde{\bm{\Sigma}}_{on}-\bm{\Sigma} &= \frac{1}{n}\Big[\sum_{i=1}^{n_0}\bm{X}_i\bm{X}_i^\top + \sum_{i=n_0+1}^n (\bm{X}_i+\widetilde{\bm{\Lambda}}\bm{\xi}_i)(\bm{X}_i+\widetilde{\bm{\Lambda}}\bm{\xi}_i)^\top\Big] -\frac{n-n_0}{4n}\widetilde{\bm{\Lambda}}^2 - \bm{\Sigma}\\
             & = \Big(\frac{1}{n}\sum_{i=1}^n \bm{X}_i\bm{X}_i^\top-\bm{\Sigma}\Big) + \frac{n-n_0}{n}\Big(\frac{1}{n-n_0}\sum_{i=n_0+1}^n\widetilde{\bm{\Lambda}}\bm{\xi}_i\bm{\xi}_i^\top\widetilde{\bm{\Lambda}}^\top-\frac{1}{4}\widetilde{\bm{\Lambda}}^2\Big) \\
             & + \frac{n-n_0}{n}\Big(\frac{1}{n-n_0}\sum_{i=n_0+1}^n \big(\bm{X}_i\bm{\xi}_i^\top\widetilde{\bm{\Lambda}}+ \widetilde{\bm{\Lambda}}\bm{\xi}_i\bm{X}_i^\top\big)\Big) : = \tilde{I}_1+ \frac{n-n_0}{n}\tilde{I}_2+ \frac{n-n_0}{n}\tilde{I}_3
         \end{aligned}
     \end{equation}
     and reach $\|\widetilde{\bm{\Sigma}}_{on}-\bm{\Sigma}\|_{op}\le \|\tilde{I}_1\|_{op}+ \|\tilde{I}_2\|_{op}+ \|\tilde{I}_3\|_{op}.$ All that remains is to bound $\|\tilde{I}_k\|_{op}$ ($k=1,2,3$) by the same arguments for bounding $\|I_k\|_{op}$ in (\ref{4.44}).

\bibliographystyle{plain}
\bibliography{libr}

\begin{thebibliography}{10}

\bibitem{abdalla2024covariance}
Pedro Abdalla and Nikita Zhivotovskiy.
\newblock Covariance estimation: Optimal dimension-free guarantees for
  adversarial corruption and heavy tails.
\newblock {\em Journal of the European Mathematical Society}, 2024.

\bibitem{chapeau2008fisher}
Fran{\c{c}}ois Chapeau-Blondeau, Solenna Blanchard, and David Rousseau.
\newblock Fisher information and noise-aided power estimation from one-bit
  quantizers.
\newblock {\em Digital Signal Processing}, 18(3):434--443, 2008.

\bibitem{chen2022quantizing}
Junren Chen, Michael~K. Ng, and Di~Wang.
\newblock Quantizing heavy-tailed data in statistical estimation: (near)
  minimax rates, covariate quantization, and uniform recovery.
\newblock {\em IEEE Transactions on Information Theory}, 70(3):2003--2038,
  2024.

\bibitem{chen2022high}
Junren Chen, Cheng-Long Wang, Michael~K. Ng, and Di~Wang.
\newblock High dimensional statistical estimation under uniformly dithered
  one-bit quantization.
\newblock {\em IEEE Transactions on Information Theory}, 69(8):5151--5187,
  2023.

\bibitem{chen2023quantized}
Junren Chen, Yueqi Wang, and Michael~K. Ng.
\newblock Quantized low-rank multivariate regression with random dithering.
\newblock {\em IEEE Transactions on Signal Processing}, 71:3913--3928, 2023.

\bibitem{choi2016near}
Junil Choi, Jianhua Mo, and Robert~W Heath.
\newblock Near maximum-likelihood detector and channel estimator for uplink
  multiuser massive mimo systems with one-bit adcs.
\newblock {\em IEEE Transactions on Communications}, 64(5):2005--2018, 2016.

\bibitem{danaee2022distributed}
Alireza Danaee, Rodrigo~C de~Lamare, and Vitor~Heloiz Nascimento.
\newblock Distributed quantization-aware rls learning with bias compensation
  and coarsely quantized signals.
\newblock {\em IEEE Transactions on Signal Processing}, 70:3441--3455, 2022.

\bibitem{dirksen2019quantized}
Sjoerd Dirksen.
\newblock Quantized compressed sensing: a survey.
\newblock In {\em Compressed Sensing and Its Applications: Third International
  MATHEON Conference 2017}, pages 67--95. Springer, 2019.

\bibitem{dirksen2022supp}
Sjoerd Dirksen and Johannes Maly.
\newblock Supplement to ``covariance estimation under one-bit quantization".
\newblock 2022.

\bibitem{dirksen2023tuning}
Sjoerd Dirksen and Johannes Maly.
\newblock Tuning-free one-bit covariance estimation using data-driven
  dithering.
\newblock {\em IEEE Transactions on Information Theory}, 70(7):5228--5247,
  2024.

\bibitem{dirksen2022covariance}
Sjoerd Dirksen, Johannes Maly, and Holger Rauhut.
\newblock Covariance estimation under one-bit quantization.
\newblock {\em The Annals of Statistics}, 50(6):3538--3562, 2022.

\bibitem{dirksen2021non}
Sjoerd Dirksen and Shahar Mendelson.
\newblock Non-gaussian hyperplane tessellations and robust one-bit compressed
  sensing.
\newblock {\em Journal of the European Mathematical Society}, 23(9):2913--2947,
  2021.

\bibitem{eamaz2021modified}
Arian Eamaz, Farhang Yeganegi, and Mojtaba Soltanalian.
\newblock Modified arcsine law for one-bit sampled stationary signals with
  time-varying thresholds.
\newblock In {\em ICASSP 2021-2021 IEEE International Conference on Acoustics,
  Speech and Signal Processing (ICASSP)}, pages 5459--5463. IEEE, 2021.

\bibitem{eamaz2022covariance}
Arian Eamaz, Farhang Yeganegi, and Mojtaba Soltanalian.
\newblock Covariance recovery for one-bit sampled non-stationary signals with
  time-varying sampling thresholds.
\newblock {\em IEEE Transactions on Signal Processing}, 70:5222--5236, 2022.

\bibitem{eamaz2023covariance}
Arian Eamaz, Farhang Yeganegi, and Mojtaba Soltanalian.
\newblock Covariance recovery for one-bit sampled stationary signals with
  time-varying sampling thresholds.
\newblock {\em Signal Processing}, 206:108899, 2023.

\bibitem{fang2010adaptive}
Jun Fang and Hongbin Li.
\newblock Adaptive distributed estimation of signal power from one-bit
  quantized data.
\newblock {\em IEEE Transactions on Aerospace and Electronic Systems},
  46(4):1893--1905, 2010.

\bibitem{freedman2009statistical}
David~A Freedman.
\newblock {\em Statistical models: theory and practice}.
\newblock cambridge university press, 2009.

\bibitem{gray1993dithered}
Robert~M Gray and Thomas~G Stockham.
\newblock Dithered quantizers.
\newblock {\em IEEE Transactions on Information Theory}, 39(3):805--812, 1993.

\bibitem{hanna2021quantization}
Osama~A Hanna, Yahya~H Ezzeldin, Christina Fragouli, and Suhas Diggavi.
\newblock Quantization of distributed data for learning.
\newblock {\em IEEE Journal on Selected Areas in Information Theory},
  2(3):987--1001, 2021.

\bibitem{jacovitti1994estimation}
Giovanni Jacovitti and Alessandro Neri.
\newblock Estimation of the autocorrelation function of complex gaussian
  stationary processes by amplitude clipped signals.
\newblock {\em IEEE transactions on information theory}, 40(1):239--245, 1994.

\bibitem{jayant1972application}
NS~Jayant and LR~Rabiner.
\newblock The application of dither to the quantization of speech signals.
\newblock {\em Bell System Technical Journal}, 51(6):1293--1304, 1972.

\bibitem{jeong2022sub}
Halyun Jeong, Xiaowei Li, Yaniv Plan, and Ozgur Yilmaz.
\newblock Sub-gaussian matrices on sets: Optimal tail dependence and
  applications.
\newblock {\em Communications on Pure and Applied Mathematics},
  75(8):1713--1754, 2022.

\bibitem{jolliffe2002principal}
Ian~T Jolliffe.
\newblock {\em Principal component analysis for special types of data}.
\newblock Springer, 2002.

\bibitem{jung2021quantized}
Hans~Christian Jung, Johannes Maly, Lars Palzer, and Alexander Stollenwerk.
\newblock Quantized compressed sensing by rectified linear units.
\newblock {\em IEEE transactions on information theory}, 67(6):4125--4149,
  2021.

\bibitem{koltchinskii2017concentration}
Vladimir Koltchinskii and Karim Lounici.
\newblock Concentration inequalities and moment bounds for sample covariance
  operators.
\newblock {\em Bernoulli}, pages 110--133, 2017.

\bibitem{ledoit2003improved}
Olivier Ledoit and Michael Wolf.
\newblock Improved estimation of the covariance matrix of stock returns with an
  application to portfolio selection.
\newblock {\em Journal of empirical finance}, 10(5):603--621, 2003.

\bibitem{limb1969design}
JO~Limb.
\newblock Design of dither waveforms for quantized visual signals.
\newblock {\em The Bell System Technical Journal}, 48(7):2555--2582, 1969.

\bibitem{12-BEJHDcovariance}
Karim Lounici.
\newblock {High-dimensional covariance matrix estimation with missing
  observations}.
\newblock {\em Bernoulli}, 20(3):1029 -- 1058, 2014.

\bibitem{maly2022new}
Johannes Maly, Tianyu Yang, Sjoerd Dirksen, Holger Rauhut, and Giuseppe Caire.
\newblock New challenges in covariance estimation: multiple structures and
  coarse quantization.
\newblock In {\em Compressed Sensing in Information Processing}, pages 77--104.
  Springer, 2022.

\bibitem{roberts1962picture}
Lawrence Roberts.
\newblock Picture coding using pseudo-random noise.
\newblock {\em IRE Transactions on Information Theory}, 8(2):145--154, 1962.

\bibitem{schuchman1964dither}
Leonard Schuchman.
\newblock Dither signals and their effect on quantization noise.
\newblock {\em IEEE Transactions on Communication Technology}, 12(4):162--165,
  1964.

\bibitem{sun2022quantized}
Zhongxing Sun, Wei Cui, and Yulong Liu.
\newblock Quantized corrupted sensing with random dithering.
\newblock {\em IEEE Transactions on Signal Processing}, 70:600--615, 2022.

\bibitem{thrampoulidis2020generalized}
Christos Thrampoulidis and Ankit~Singh Rawat.
\newblock The generalized lasso for sub-gaussian measurements with dithered
  quantization.
\newblock {\em IEEE Transactions on Information Theory}, 66(4):2487--2500,
  2020.

\bibitem{tropp2015introduction}
Joel~A Tropp et~al.
\newblock An introduction to matrix concentration inequalities.
\newblock {\em Foundations and Trends{\textregistered} in Machine Learning},
  8(1-2):1--230, 2015.

\bibitem{van1966spectrum}
John~H Van~Vleck and David Middleton.
\newblock The spectrum of clipped noise.
\newblock {\em Proceedings of the IEEE}, 54(1):2--19, 1966.

\bibitem{vershynin_2012}
Roman Vershynin.
\newblock {\em Introduction to the non-asymptotic analysis of random matrices},
  page 210–268.
\newblock Cambridge University Press, 2012.

\bibitem{vershynin2018high}
Roman Vershynin.
\newblock {\em High-dimensional probability: An introduction with applications
  in data science}, volume~47.
\newblock Cambridge university press, 2018.

\bibitem{xu2020quantized}
Chunlei Xu and Laurent Jacques.
\newblock Quantized compressive sensing with rip matrices: The benefit of
  dithering.
\newblock {\em Information and Inference: A Journal of the IMA}, 9(3):543--586,
  2020.

\bibitem{yang2023plug}
Tianyu Yang, Johannes Maly, Sjoerd Dirksen, and Giuseppe Caire.
\newblock Plug-in channel estimation with dithered quantized signals in
  spatially non-stationary massive mimo systems.
\newblock {\em IEEE Transactions on Communications}, 72(1):387--402, 2024.

\bibitem{zhang2017zipml}
Hantian Zhang, Jerry Li, Kaan Kara, Dan Alistarh, Ji~Liu, and Ce~Zhang.
\newblock Zipml: Training linear models with end-to-end low precision, and a
  little bit of deep learning.
\newblock In {\em International Conference on Machine Learning}, pages
  4035--4043. PMLR, 2017.

\end{thebibliography}
\end{document}